\documentclass[runningheads]{llncs}

\usepackage{times}
\usepackage[utf8]{inputenc} 
\usepackage[T1]{fontenc}    
\usepackage{hyperref}       
\usepackage{url}            
\usepackage{amsfonts}       
\usepackage{nicefrac}       
\usepackage{microtype}      
\usepackage{xcolor}
\usepackage{graphicx}
\usepackage{makecell}
\usepackage[ruled,vlined,linesnumbered]{algorithm2e}
\usepackage[caption=false,font=footnotesize]{subfig}
\usepackage{amsmath,amssymb}
\hypersetup{
	colorlinks   = true,
	urlcolor    = teal,
	citecolor = teal,
	linkcolor = teal
}
\newcommand{\citet}[1]{\citeauthor{#1}~\shortcite{#1}}

\begin{document}

\title{Learning and Planning in the \\Feature Deception Problem}

\author{Zheyuan Ryan Shi\inst{1} \and
Ariel D. Procaccia\inst{2} \and
Kevin S. Chan\inst{3}\and
Sridhar Venkatesan\inst{4}\and
Noam Ben-Asher\inst{3}\and
Nandi O. Leslie\inst{3}\and
Charles Kamhoua\inst{3}\and
Fei Fang\inst{1}
}

\authorrunning{Z. R. Shi et al.}
\institute{Carnegie Mellon University \and
Harvard University \and
Army Research Laboratory \and
Perspecta Labs}

\maketitle

\begin{abstract}
    Today's high-stakes adversarial interactions feature attackers who constantly breach the ever-improving security measures.
    Deception mitigates the defender's loss by misleading the attacker to make suboptimal decisions.
    In order to formally reason about deception, we introduce the \emph{feature deception problem (FDP)}, a domain-independent model and present a learning and planning framework for finding the optimal deception strategy, taking into account the adversary's preferences which are initially unknown to the defender. 
    We make the following contributions. (1) We show that we can uniformly learn the adversary's preferences using data from a modest number of deception strategies.
    (2) We propose an approximation algorithm for finding the optimal deception strategy given the learned preferences and show that the problem is NP-hard. 
    (3) We perform extensive experiments to validate our methods and results.
    In addition, we provide a case study of the credit bureau network to illustrate how FDP implements deception on a real-world problem.
\end{abstract}

\section{Introduction}

The world today poses more challenges to security than ever before.
Consider the cyberspace or the financial world where a defender is protecting a collection of targets, e.g.\ servers or accounts. Despite the ever-improving security measures, malicious attackers work diligently and creatively to outstrip the defense~\cite{potter2009CFS}. Against an attacker with previously unseen exploits and abundant resources, the attempt to protect any target is almost surely a lost cause~\cite{hurlburt2016Computer}. However, the defender could induce the attacker to attack a less harmful, or even fake, target. This can be seen as a case of deception.

Deception has been an important tactic in military operations for millenia~\cite{latimer2001deception}.
More recently, it has been extensively studied in cybersecurity~\cite{jajodia2016Springer,horak2017Gamesec}. At the start of an attack campaign, attackers typically perform reconnaissance to learn the configuration of the machines in the network using tools such as Nmap~\cite{lyon2009nmap}. Security researchers have proposed many deceptive measures to manipulate a machine's response to these probes~\cite{jajodia2017IEEE,albanese2016CD}, which could confound and mislead an attempt to attack.
In addition, honey-X, such as honeypots, honey users, and honey files have been developed to attract the attackers to attack these fake targets~\cite{spitzner2003IEEE}. For example, it is reported that country A once created encrypted but fake files with names of country B's military systems and marked them to be shared with country A's intelligence agency~\cite{nakashima2013}. Using sensitive filenames as bait, country A successfully lured country B's hackers to these decoy targets.

Be it commanding an army or protecting a computer network, a common characteristic is that the attacker gathers information about the defender's system to make decisions, and the defender can (partly) control how her system appears to the surveillance. We formalize this view, abstract the collected information about the defender's system that is relevant to attacker's decision-making as features, and propose the \emph{feature deception problem (FDP)} to model the strategic interaction between the defender and the attacker. 

It is evident that the FDP model could be applied to many domains by appropriately defining the relevant set of features. To be concrete, we will ground our discussion in cybersecurity, where an attacker observes the features of each network node when attempting to fingerprint the machines (example features shown in the left column of Table~\ref{tab:feature}) and then chooses a node to compromise. 
Attackers may have different preferences over feature value combinations when choosing targets to attack.
If an intruder has an exploit for Windows machines, a Linux server might not be attractive. If the attacker is interested in exfiltration, he might choose a machine running database services.
If the defender knows the attacker's preferences, she could strategically configure important machines appear undesirable or configure the honeypots to appear attractive to the attacker, by changing the observed value of the features, e.g.\ Table~\ref{tab:feature}. However, to make an informed decision, she needs to first learn the attacker's preferences.

\begin{table}[t] 
    \centering
    \begin{tabular}{c c c}
        \hline \textbf{Feature} & \textbf{Observed value} & \textbf{Actual value}\\ \hline
        Operating system & Windows 2016 & RHEL 7\\ 
        Service version & v1.2 & v1.4\\
        IP address & 10.0.1.2 & 10.0.2.1\\ 
        Open ports & 22, 445  & 22, 1433\\ 
        Round trip time for probes  \cite{shamsi2014ACM} & 16 ms & 84 ms \\ \hline
    \end{tabular}
            \caption{Example features in cybersecurity}
                \label{tab:feature} 
\end{table}

\textbf{Our Contributions\,\,} Based on our proposed FDP model, we provide a learning and planning framework and make three key contributions. 
First, we analyze the sample complexity of learning attacker's preferences. We prove that to learn a classical subclass of preferences that is typically used in the inverse reinforcement learning and behavioral game theory literature, the defender needs to gather only a polynomial number of data points on a linear number of feature configurations.
The proof leverages what we call the \textit{inverse feature difference} matrix (IFD), and shows that the complexity depends on the norm of this matrix. 
If the attacker is aware of the learning, they may try to interfere with the learning process by launching the data-poisoning attack, a typical threat model in adversarial machine learning. Using the IFD, we demonstrate the robustness of learning in FDP against this kind of attack.
Second, we study the planning problem of finding the optimal deception strategy against learned attacker's preferences. We show that it is NP-hard and propose an approximation algorithm.
In addition, we perform extensive experiments to validate our results. We also conduct a case study to illustrate how our FDP framework implements deception on the network of a credit bureau.

\section{The Feature Deception Problem}
In an FDP, a defender aims to protect a set $N$ of $n$ targets from an adversary. Each target $i \in N$ has a set $M$ of $m$ features.
The adversary observes these features and then chooses a target to attack. The defender incurs a loss $u_i \in [-1,1]$ if the adversary chooses to attack target $i$.\footnote{Typically, the loss $u_{i}$ is non-negative, but it might be negative if, for example, the target is set up as a decoy or honeypot, and allows the defender to gain information about the attacker.} 
The defender's objective is to minimize her expected loss.
Now, we introduce several key elements in FDP. 
We provide further discussions on some of the assumptions in FDP in the final section.

\paragraph{Features}
Features are the key element of the FDP model. Each feature
has an \textit{observed} value and an \textit{actual} value. The actual value is given and fixed, while the defender can manipulate the observed value. Only the observed values are visible to the adversary. This ties into the notion of deception, where one may think of the actual value as representing the ``ground truth'' whereas the observed value is what the defender would like the attacker to see. 
Since deception means manipulating the attacker's perceived value of a target, not the actual value, changing the observable values does not affect the defender's loss $u_i$ at each target.

Table~\ref{tab:feature} shows an example in cybersecurity. In practice, there are many ways to implement deception. 
For example, a node running Windows (actual feature) manages to reply to reconnaissance queries in Linux style using tools like OSfuscate. Then the attacker might think the node is running Linux (observed feature).
For IP deception, Jafarian et al.~\cite{jafarian2012openflow} and Chiang et al.~\cite{chiang2016acyds} demonstrate methods to present to the attacker a different IP from the actual one. 
In addition, when we ``fake open'' a port with no real vulnerable service runs on it, an attack on the underlying service will fail. This could be done with command line tools or existing technologies like Honeyd~\cite{provos2004virtual}.

\paragraph{Feature representation}
We represent the observed feature values of target $i$ by a vector $x_i = (x_{ik})_{k \in M} \in [0,1]^m$. We denote their corresponding actual values as $\hat x_i \in [0,1]^m$. 
We allow for both continuous and discrete features. In practice, we may have categorical features, such as the type of operating system, and they can be represented using one-hot encoding with binary features. 

\paragraph{Feasibility constraints}
For a feature $k$ with actual value $\hat x_{ik}$, the defender can set its observed value $x_{ik} \in C(\hat x_{ik}) \subseteq [0,1]$, where the feasible set $C(\hat x_{ik})$ is determined by the actual value. 
For continuous features, we assume $C(\hat x_{ik})$ takes the form $[\hat x_{ik} - \tau_{ik}, \hat x_{ik} + \tau_{ik}] \cap [0,1]$ where $\tau_{ik}\in[0,1]$. 
This captures the feasibility constraint in setting up the observed value of a feature based on its actual value. 
Take the round trip time (RTT) as an example. Shamsi et al. fingerprint the OS using RTT of the SYN-ACK packets~\cite{shamsi2014ACM}. Typical RTTs are in the order of few seconds (Fig. 4~\cite{shamsi2014ACM}), while a typical TCP session is ~3-5 minutes. Thus, perturbing RTT within a few seconds is reasonable, but greater perturbation is dubious.

For binary features, $C(\hat x_{ik}) \subseteq \{0,1\}$.
In addition to these feasibility constraints for individual features, we also allow for linear constraints over multiple features, which could encode natural constraints for categorical features with one-hot encoding, e.g. $\sum_{k \in M'} x_{ik} = 1$, with $M' \subseteq M$ being the subset of features that collectively represent one categorical feature. They may also encode the realistic considerations when setting up the observed features. For example, $x_{ik_{1}} + x_{ik_{2}} \leq 1$ could mean that a Linux machine $(x_{ik_1} = 1)$ cannot possibly have ActiveX available $(x_{ik_2} = 1)$.

\paragraph{Budget constraint}
Deception comes at a cost. We assume the cost is additive across targets and features: $c = \sum_{i \in N} \sum_{k \in M} c_{ik}$, where $c_{ik} = \eta_{ik} |x_{ik} - \hat x_{ik}|$. For a continuous feature $k$, $\eta_{ik}$ represents the cost associated with unit of change from the actual value to the observable value. In the example of RTT deception, defender’s cost is the packet delay which can be considered linear. If $k$ is binary, $\eta_{ik}$ defines the cost of switching states. The defender has a budget $B$ to cover these costs. 
We note that, though we introduce these explicit forms of feasibility constraints and cost structure, our algorithms in the sequel are not specific to these forms.

\paragraph{Defender strategies}
The defender's strategy is an observed feature configuration $x = \{x_i\}_{i \in N}$. 
The defender uses only pure strategies.

\paragraph{Attacker strategies}
The attacker's pure strategy is to choose a target $i \in N$ to attack. 
Since human behavior is not perfectly rational and the attacker may have preferences that are unknown to the defender a priori, we reason about the adversary using a general class of bounded rationality models. We assume the attacker's utilities are characterized by a score function $f: [0,1]^m \to \mathbb R_{>0}$ over the observed feature values of a target. 
Given observed feature configuration $x = \{x_i\}_{i \in N}$, he attacks target $i$ with probability $\frac{f(x_{i})}{\sum_{j \in N} f(x_j)}$.
$f$ may take any form and in this paper, we assume that it can be parameterized by or approximated with a neural network with parameter $w$.
In some of the theoretical analyses, we focus on a subclass of functions
\begin{equation} \label{eqn:score} 
    f_w(x_i) = \exp\left( \sum\nolimits_{k \in M} w_k x_{ik} \right).
\end{equation}
We omit the subscript $w$ when there is no confusion. This functional form is commonly used to approximate the agent's reward or utility function in inverse reinforcement learning and behavioral game theory, and has been empirically shown to capture many attacker preferences in cybersecurity~\cite{abbasi2016know}.
For example, the tactics of advanced persistent threat group APT10~\cite{pwc2017} are driven by: (1) final goal: they aim at exfiltrating data from workstation machines; (2) expertise: they employ exploits against Windows workstations; (3) services available: their exploits operate against file sharing and remote desktop services. Thus, APT10 prefer to attack machines with Windows OS running a file-sharing service on the default port. Each of these properties is a ``feature'' in FDP and a score function $f$ in Eq~\eqref{eqn:score} can assign a greater weight for each of these features.
It can also capture more complex preferences by using hand-crafted features based on domain knowledge.
For example, APT10 typically scan for NetBIOS services (i.e., ports 137 and 138), and Remote Desktop Protocol services (i.e., ports 445 and 3389) to identify systems that they might get onto~\cite{pwc2017}. Instead of treating the availability of ports as features, we may design a binary feature indicating whether each of the service is available (representing an ``OR'' relationship of the port availability features). 
We also show a more efficient way to approximately handle combinatorial preferences in Section~\ref{sec:casestudy}.
In addition, this score function also captures fully rational attackers in the limit.

The ultimate goal of the defender is to find the optimal feature configuration against an unknown attacker. This can be decomposed into two subtasks: \textit{learning} the attacker's behavior model from attack data and \textit{planning} how to manipulate the feature configuration to minimize her expected loss based on the learned preferences. In the following sections, we first analyze the sample complexity of the learning task and then propose algorithms for the planning task.

\section{Learning the Adversary's Preferences} \label{sec:samplecomplexity}
The defender learns the adversary's score function $f$ from a set of $d$ labeled data points each in the format of $(N, x, y)$ where $N$ is the set of targets and $x$ is the observed feature configuration of all targets in $N$. The label $y \in N$ indicates that the adversary attacks target $y$.

In practice, there are two ways to carry out the learning stage. First, the defender can learn from historical data.
Second, the defender can also actively collect data points while manipulating the observed features of the network. This is often done with honeynets~\cite{spitzner2003IEEE}, i.e. a network of honeypots.

No matter which learning mode we use, it is often the case, e.g. in cybersecurity, that the dataset contains multiple data points with the same $x$, since changing the defender configuration frequently leads to too much overhead.
In addition, at the learning stage, only the observed feature values $x$ matter because the attacker does not observe the actual feature values $\hat x$. 
The feasibility constraints $C(\hat x_{ik})$ on each feature still apply. Yet, they are irrelevant during learning because we use either historical data that satisfy these constraints, or honeypots for which these constraints are vacuous.

To analyze the sample complexity of learning the adversary's preferences, we focus on the classical form score function $f$ in Eq~\eqref{eqn:score}.
We show that, in an FDP with $m$ features,
the defender can learn the attacker's behavior model correctly with high probability, using only $m$ observed feature configurations and a polynomial number of samples.
We view this condition as very mild, because even if the network admin's historical dataset does not meet the requirement,
she could set up a honeynet to elicit attacks, where she can control the feature configurations of each target~\cite{spitzner2003IEEE}. It is still not free for the defender to change configurations, but attacks on honeynet do not lead to actual loss since it runs in parallel with the production network.

 To capture the multiple features in FDP, we introduce the \textit{inverse feature difference matrix} $(A^{st})^{-1}$.
Specifically, given observed feature configurations $x^1,\dots, x^m$, for any two targets $s, t \in N$, let $A^{st}$ be the $m \times m$ matrix whose $(i,j)$-entry is $a^{st}_{ij} = x^i_{sj} - x^i_{tj}$. $A^{st}$ captures the matrix-level correlation among feature configurations. We use the matrix norm of $(A^{st})^{-1}$ to bound the learning error. 

For feature configuration $x$, let $D^x(t) = \frac{f(x_t)}{\sum_{i \in N} f(x_i)}$ be the attack probability on target $t$. We assume $\rho := \min_{x, t} D^x(t) > 0$. Let $\alpha = \min_{s \neq t} ||(A^{st})^{-1}||$, where $||\cdot||$ is the matrix norm induced by the $L^1$ vector norm, i.e.\ $||(A^{st})^{-1}|| = \sup_{y \neq 0} \frac{|(A^{st})^{-1} y|}{|y|}$. Our result is stated as the following theorem.

\begin{theorem} \label{thm:samplecomplexity}
Consider $m$ observed feature configurations $x^1, x^2,\dots,$ $ x^m \in [0,1]^{mn}$. 
With $\Omega(\frac{\alpha^4 m^4}{\rho \epsilon^2} \log \frac{nm}{\delta})$ samples for each of the $m$ feature configurations, with probability $1-\delta$, we can learn a score function $\hat f(\cdot)$ with uniform multiplicative error $\epsilon$ of the true $f(\cdot)$, i.e., $\frac{1}{1 + \epsilon} \leq \frac{f(x_i)}{\hat f(x_i)} \leq 1 + \epsilon, \forall x_i$.
\end{theorem}

\begin{proof}
Let $\hat D^x(t) = \frac{\hat f(x_t)}{\sum_{i \in N} \hat f(x_i)}$. We leverage a known result from behavioral game theory~\cite{haghtalab2016IJCAI}. It cannot be directly translated to sample complexity guarantee in FDP because of the correlation among feature configurations, but we use it to reason about attack probabilities in proving Theorem~\ref{thm:samplecomplexity}. 
\begin{lemma}\cite{haghtalab2016IJCAI} \label{lem:distribution}
	Given observable features $x \in [0,1]^{mn}$, and $\Omega(\frac{1}{\rho \epsilon^2} \log\frac{n}{\delta})$ samples, we have $\frac{1}{1+\epsilon} \leq \frac{\hat D^x(t)}{D^x(t)} \leq 1+\epsilon$ with probability $1-\delta$, for all $t \in N$.
\end{lemma}

	Fix $\epsilon, \delta > 0$. From Eq.~\eqref{eqn:score}, for each $x^i$ where $i = 1,2,\dots, m$, we have
	\[
	\sum_{j=1}^m w_j (x^i_{sj} - x^i_{tj}) = \ln \frac{D^{x^i}(s)}{D^{x^i}(t)}, \quad \forall s, t \in N, s \neq t
	\]
	Let 
	\[b^{st} = (\ln \frac{D^{x^1}(s)}{D^{x^1}(t)}, \dots, \ln \frac{D^{x^m}(s)}{D^{x^m}(t)})^T.\] 
	The system of equations above can be represented by $A^{st} w = b^{st}$.
	It is known that $||A^{st}|| = \max_{1 \leq j \leq m} \sum_{i=1}^m |a^{st}_{ij}|$. In our case, the feature values are bounded in $[0,1]$ and thus $|a^{st}_{ij}| \leq 1$. This yields $||A^{st}|| \leq m$.
	Now, choose $s, t$ such that $||(A^{st})^{-1}|| = \alpha$. Suppose $A^{st}$ is invertible.
	
	Let $\epsilon' = \frac{\epsilon}{4\alpha^2 m^2}$ and $\delta' = \frac{\delta}{m}$. Suppose we have $\Omega(\frac{1}{\rho \epsilon'^2} \log \frac{n}{\delta'})$ samples. From Lemma~\ref{lem:distribution}, for any node $r \in N$ and any feature configuration $x^i$ where $i=1,2,\dots, m$, $\frac{1}{1+\epsilon'} \leq \frac{\hat D^{x^i}(r)}{D^{x^i}(r)} \leq 1+\epsilon'$ with probability $1 - \delta'$. The bound holds for all strategies simultaneously with probability at least $1-m\delta' = 1-\delta$, using a union bound argument. In particular, for our chosen nodes $s$ and $t$, we have
	\[
	\frac{1}{(1+\epsilon')^2} \leq \frac{\hat D^{x^i}(s)}{\hat D^{x^i}(t)} \frac{D^{x^i}(t)}{D^{x^i}(s)}\leq (1+\epsilon')^2, \quad \forall i = 1,\dots, m
	\]
	
	Define $\hat{b}^{st}$ similarly as $b^{st}$ but using empirical distribution $\hat D$ instead of true distribution $D$. Let $e = \hat{b}^{st} - b^{st}$. Then, for each $i = 1,\dots, m$, we have 
	\[
		\hspace{-20pt}
	-2\epsilon' \leq 2\ln \frac{1}{1+\epsilon'}\leq e_i = \ln \frac{\hat D^{x^i}(s) D^{x^i}(t)}{\hat D^{x^i}(t) D^{x^i}(s)} \leq 2\ln (1+\epsilon') \leq 2\epsilon'
	\]
	Therefore, we have $|e| \leq 2\epsilon'm$. Let $\hat w$ be such that $A^{st} \hat w = \hat{b}^{st}$, i.e. $\hat w - w = (A^{st})^{-1} e$. Observe that
	\[
	\begin{split}
	&\frac{|(A^{st})^{-1} e|/|(A^{st})^{-1} b^{st}|}{|e|/|b^{st}|} \leq \max_{\tilde e, \tilde b^{st} \neq 0} \frac{|(A^{st})^{-1} \tilde e|/|(A^{st})^{-1} \tilde b^{st}|}{|\tilde e|/|\tilde b^{st}|}\\
	&= \max_{\tilde e\neq 0} \frac{|(A^{st})^{-1} \tilde e|}{|\tilde e|} \max_{\tilde b^{st} \neq 0} \frac{|\tilde b^{st}|}{|(A^{st})^{-1} \tilde b^{st}|} = \max_{\tilde e\neq 0} \frac{|(A^{st})^{-1} \tilde e|}{|\tilde e|} \max_{y \neq 0} \frac{|A^{st} y|}{|y|} = ||(A^{st})^{-1}|| \cdot ||A^{st}||
	\end{split}
	\]
	This leads to
	\[
	\begin{split}
	|(A^{st})^{-1} e| &\leq ||(A^{st})^{-1}|| \cdot ||A^{st}|| \cdot |e| \cdot \frac{|(A^{st})^{-1} b^{st}|}{|b^{st}|} \\
	&\leq ||(A^{st})^{-1}|| \cdot ||A^{st}|| \cdot |e| \cdot \max_{\tilde b^{st} \neq 0} \frac{|(A^{st})^{-1} \tilde b^{st}|}{|\tilde b^{st}|} \\
	&= ||(A^{st})^{-1}||^2 \cdot ||A^{st}|| \cdot |e| \leq \alpha^2 m (2 \epsilon' m)
	\end{split}
	\]
	For any observable feature configuration $x$,
	\[
	\begin{split}
	& \left|\left(\sum_{j=1}^m w_j x_{ij}\right) - \left(\sum_{j=1}^m \hat w_j x_{ij}\right) \right| \leq \sum_{j=1}^m |\hat{w}_j - w_j| = |(A^{st})^{-1} e| \leq \alpha^2 m (2 \epsilon' m) = \frac{\epsilon}{2}
	\end{split}
	\]
	Therefore,
	\[
	\frac{1}{1 + \epsilon} \leq \frac{f(x_i)}{\hat f(x_i)} \leq 1 + \epsilon.\qed
	\] 
\end{proof}
It is easy to see that we do not have to use the same pair of targets $(s,t)$ for every feature configuration. In fact, this result can be easily adapted to allow for each feature configuration being implemented on a different system with a different set and number of targets. Instead of defining $A^{st}$ and $b^{st}$, we could define $A$ and $b$, where row $i$ of $A$ and $i$-th entry of $b$ correspond to feature configuration $x^i$ and targets $(s^i, t^i)$. If feature configuration $x^i$ is implemented on a system with $n_i$ targets, we need $\Omega(\frac{1}{\rho \epsilon'^2} \log \frac{n_i}{\delta'})$ samples from this system, and then the argument above still holds.

The $\alpha$ in Theorem~\ref{thm:samplecomplexity} need not be large, especially if the defender can select the feature configurations to collect data and elicit preferences. Consider a sequence of $m$ feature configurations $x^1, \dots, x^m$, and focus on targets 1 and 2. For each $x^j$, let the features on target 1 be identical to target 2, except for the $j$-th feature, where $x^j_{1j} = 1$ and $x^j_{2j} = 0$.
This leads to $A^{12} = I$, and $\alpha \leq 1$.
This also shows that it is not hard to set up the configurations such that $A^{st}$ is nonsingular.

An adversary who is aware of the defender's learning procedure might sometimes intentionally attack without following his true score function $f$, to mislead the defender. The following theorem states that the defender can still learn an approximately correct $f$ even if the attacker contaminates a $\gamma$ fraction of the data.

\begin{theorem} \label{thm:poison}
In the setting of Theorem~\ref{thm:samplecomplexity},
if the attacker modifies a $\gamma \leq \frac{\epsilon \rho}{4\alpha m}$ fraction of the data points for each feature configuration, the function $f$ can be learned within multiplicative error $3\epsilon$.
\end{theorem}

\begin{proof}
Fix two nodes $s, t$. Recall that in Theorem~\ref{thm:samplecomplexity}, without data poisoning, we learned the weights $w$ by solving the linear equations $A^{st} \tilde w = \tilde b^{st}$ based on the empirical distribution of attacks, where $\tilde b^{st} = (\ln \frac{\tilde D^{x^1}(s)}{\tilde D^{x^1}(t)}, \dots, \ln \frac{\tilde D^{x^m}(s)}{\tilde D^{x^m}(t)})$.
\footnote{Refer to the proof of Theorem~\ref{thm:samplecomplexity} for the notations used.} 
Denote a parallel system of equations $A^{st} \hat w = \hat b^{st}$ which uses the poisoned data. We are interested in bounding $|\hat w - \tilde w| = |(A^{st})^{-1} (\hat b^{st} - \tilde b^{st})|$. Consider the $k$-th entry in the vector $\hat b^{st} - \tilde b^{st}$:
\[
|(\hat b^{st} - \tilde b^{st})_k| = \left|\ln \frac{\hat D^{x^k}(s)}{\hat D^{x^k}(t)} \frac{\tilde D^{x^k}(t)}{\tilde D^{x^k}(s)}\right|
\]
To simplify the notations, we denote $\tilde D^{x^k}(t) = \gamma_t^k$ and $\tilde D^{x^k}(s) = \gamma_s^k$, and without loss of generality, assume $\gamma_t^k \leq \gamma_s^k$. To find an upper bound of RHS of the above equation, we define function $g(\gamma_1, \gamma_2) = \frac{\gamma_t^k (\gamma_s^k + \gamma_1)}{\gamma_s^k (\gamma_t^k - \gamma_2)}$, and define function $h(\gamma_1, \gamma_2) = |\ln g(\gamma_1, \gamma_2)|$. The constraint that the attacker can only change $\gamma$ fraction of the points translates into $|\gamma_1|, |\gamma_2|, |\gamma_1 - \gamma_2| \leq \gamma$. Since $g$ is increasing in $\gamma_1$ and $\gamma_2$, $g$ attains maximum at $(\gamma_1, \gamma_2) = (\gamma, \gamma)$ and minimum at $(\gamma_1, \gamma_2) = (-\gamma, -\gamma)$, which are the only two possible maxima of $h$. Observe that $g(\gamma, \gamma) \geq 1$ and $g(-\gamma, -\gamma) \leq 1$. It then suffices to compare $g(\gamma, \gamma)$ with $1/g(-\gamma, -\gamma)$:
\[
\hspace{-10pt}
\frac{1/g(-\gamma, -\gamma)}{g(\gamma, \gamma)} = \frac{\gamma_s (\gamma_t + \gamma)}{\gamma_t (\gamma_s - \gamma)} \frac{\gamma_s (\gamma_t - \gamma)}{\gamma_t (\gamma_s + \gamma)}
= \frac{\gamma_s^2 \gamma_t^2- \gamma_s^2 \gamma^2}{\gamma_t^2 \gamma_s^2 - \gamma_t^2 \gamma^2}
\leq 1
\]
Therefore, $h(\gamma_1, \gamma_2)$ is maximized at $(\gamma_1, \gamma_2) = (\gamma, \gamma)$. From here, we obtain
\[
\begin{split}
|(\hat b^{st} - \tilde b^{st})_k| 
&\leq \ln \frac{(\gamma_s^k + \gamma)\gamma_t^k}{(\gamma_t^k - \gamma)\gamma_s^k} = \ln \left( \left( 1 + \frac{\gamma}{\gamma_s^k} \right) \left(1 + \frac{\gamma}{\gamma_t^k - \gamma}\right) \right) \leq \frac{\gamma}{\gamma_s^k} + \frac{\gamma}{\gamma_t^k - \gamma}.
\end{split}
\]
Recall that
\[
\frac{ \left|(A^{st})^{-1} (\hat b^{st} - \tilde b^{st}) \right|}{\left|\hat b^{st} - \tilde b^{st}\right|} \leq \sup_{y \neq 0} \frac{ \left|(A^{st})^{-1} y \right|}{\left|y\right|} = ||(A^{st})^{-1}|| = \alpha
\]
Thus, we get
\[
\begin{split}
|\hat w - \tilde w| &= |(A^{st})^{-1} (\hat b^{st} - \tilde b^{st})| \leq \alpha \left|\hat b^{st} - \tilde b^{st}\right| \leq \alpha \sum_{k=1}^m \left(\frac{\gamma}{\gamma_s^k} + \frac{\gamma}{\gamma_t^k - \gamma}\right)
\end{split}
\]
Note that by Lemma~\ref{lem:distribution}, we have $\gamma_t^k \geq \frac{\rho}{1 + \epsilon'} \geq \frac{\rho}{2}$. Since we assumed that $\gamma \leq \frac{\epsilon\rho}{4\alpha m} \leq \frac{\epsilon\rho}{4}$, we know that $\gamma \leq \gamma_t/2$. Thus, we get
\[
\hspace{-10pt}
|\hat w - \tilde w| \leq \alpha \sum_{k=1}^m \left(\frac{\gamma}{\gamma_s^k} + \frac{2\gamma}{\gamma_t^k}\right) 
\leq   \frac{3\epsilon (1+\epsilon')}{4} \leq \frac{3}{4}\epsilon \left(1+\frac{1}{4}\epsilon\right)
\]
From here, using the triangle inequality, we have
\[
|\hat w - w| \leq |\hat w - \tilde w| + |\tilde w - w| \leq \frac{3}{4}\epsilon \left(1+\frac{1}{4}\epsilon\right) + \frac{\epsilon}{2} \leq \frac{3}{2}\epsilon
\]
Thus, in the end, we get
\[
\frac{1}{1+3\epsilon} \leq \frac{f(x_i)}{\hat f(x_i)} \leq 1 + 3\epsilon. \qed
\]
\end{proof}

For a general score function $f_w$, gradient-based optimizer, e.g.\ RMSProp~\cite{hinton2012rmsprop} can be applied to learn $w$ through maximum-likelihood estimation.
\[w = \arg\max_{w'} \sum\nolimits_{j\in[d]} \left[L^j_{w'}(N^j, x^j, y^j) \right]\]
\[L^j_{w'}(N^j, x^j, y^j)= \log (f_{w'}(x^j_{y^j})) - \log (\sum\nolimits_{i \in N^j} f_{w'}(x^j_i))
\]
However, it is not guaranteed to find the optimal solution given the non-convexity of $L$. 

\section{Computing the Optimal Feature Configuration}
We now embark on our second task: 
assuming the (learned) adversary's behavior model, compute the optimal observed feature configuration to minimize the defender's expected loss.
For any score function, the problem can be formulated as the following mathematical program (MP).
\begin{align}
    \min_{x} & \quad  \frac{\sum_{i \in N} f (x_i) u_i}{\sum_{i \in N} f(x_i)}  \label{geneq1} \\
    s.t. & \quad \sum_{i \in N} \sum_{k \in M} \eta_{ik} |x_{ik} - \hat x_{ik}| \leq B \label{geneq2}\\
    & \quad \mbox{Categorical feature constraints}  \label{geneq3}\\
    & \quad x_{ik} \in C(\hat x_{ik}) \qquad \qquad \forall i \in N, k \in M \label{geneq4}
\end{align}
This MP is typically non-convex and very difficult to solve. We show that the decision version of FDP is NP-complete. Hence, finding the optimal feature configuration is NP-hard. In fact, this holds even when there is only a single binary feature and the score function $f$ takes the form in Eq.~\eqref{eqn:score}.

\begin{theorem} \label{thm:hardness}
FDP is NP-complete.
\end{theorem} 
\begin{proof}
We reduce from the Knapsack problem: 
given $v \in [0, 1]^n$, $\omega \in \mathbb R^n_+, \Omega, V \in \mathbb R_+$, decide whether there exists $y \in \{0,1\}^n$ such that $\sum_{i=1}^n v_i y_i \geq V$ and $\sum_{i=1}^n \omega_i y_i \leq \Omega$.

We construct an instance of FDP. Let the set of targets be $N = \{1,\dots, n + 1\}$, and let there be a single binary feature, i.e. $M = \{1\}$ and $x_{i1} \in \{0,1\}$ for each $i \in N$. Since there is only one feature, we abuse the notation by using $x_i = x_{i1}$. Suppose each target's actual value of the feature is $\hat x_{i} = 0$. Consider a score function $f$ with $f(0) = 1$ and $f(1) = 2$. For each $i \in N$, let $u_i = (1-v_i)/\delta$ if $i \neq n+1$, and $u_{n+1} = (1 + V + \sum_{i=1}^n v_i)/\delta$. Choose a large enough $\delta \geq 1$ so that $u_{n+1} \leq 1$. For each $i\in N$, let $\eta_i = \omega_i$ if $i \neq n+1$, and $\eta_{n+1} = 0$. Finally, let the budget $B = \Omega$.

For a solution $y$ to a Knapsack instance, we construct a solution $x$ to the above FDP where $x_{i} = y_i$ for $i \neq n+1$, and $x_{n+1} = 0$. 
We know $\sum_{i\in N} \eta_i |x_i - \hat x_i| = \sum_{i\in N} \eta_i x_i \leq B$ if and only if $\sum_{i=1}^n \omega_i y_i \leq \Omega$.
Since $f(x_i) > 0$ for all $x_i$, $\frac{\sum_{i\in N} f(x_{i}) u_i}{\sum_{i\in N} f(x_{i})} \leq 1/\delta$ if and only if $\sum_{i \in N} (1-\delta u_i) f(x_i) \geq 0$. Note that $\sum_{i \in N} (1-\delta u_i) = \sum_{i=1}^n v_i (y_i + 1) - \sum_{i=1}^n v_i - V$. Thus, $y$ is a certificate of Knapsack if and only if $x$ is feasible for FDP and the defender's expected loss is at most $1/\delta$.
\end{proof}

Despite the negative results for the general case, we design an approximation algorithm for the classical score function in Eq.~\eqref{eqn:score} based on mixed integer linear programming (MILP) enhanced with binary search. As shown in Sec \ref{sec:experiments}, it can solve medium sized problems (up to 200 targets) efficiently. 
Given $f(x_i) = \exp(\sum_{k \in M} w_k x_{ik})$, scaling the score by a factor of $e^{-W}$ does not affect the attack probability, where $W = |w|$ is the $L^1$ norm of $w = (w_1,\dots, w_m)$. Thus, we treat the score function as $f(x_i) = \exp(\sum_{k \in M} w_k x_{ik} - W)$.

With slight abuse of notation, we denote the score of target $i$ as $f_i$.
Let $z_i = \sum_{k \in M} w_k x_{ik} - W \in [-2W, 0]$. We divide the interval $[-2W, 0]$ into $2W/\epsilon$ subintervals, each of length $\epsilon$. On interval $[-l\epsilon, -(l-1)\epsilon]$ with $l=0,1,\dots, 2W/\epsilon$, we approximate the function $e^{z_i}$ with the line segment of slope $\gamma_l$ connecting the points $(-l\epsilon, e^{-l\epsilon})$ and $(-(l-1)\epsilon, e^{-(l-1)\epsilon})$. We use this method to approximate $f_i$ in the following mathematical program $\mathcal{MP}1$. We represent $z_i = -\sum_{l} z_{il}$, where each variable $z_{il}$ indicates the quantity $z_i$ takes up on the interval $[-l\epsilon, -(l-1)\epsilon]$.
The constraints in Eq.~\eqref{milpeq4}-\eqref{milpeq5} ensure that $z_{i(l+1)} > 0$ only if $z_{il} = \epsilon$. While $\mathcal{MP}1$ is not technically a MILP, we can linearize the objective and the constraint involving absolute value following a standard procedure~\cite{stancu2012fractional}.
The full MILP formulation can be found in Appendix~\ref{app:algs}.

\begin{align}
    (\mathcal{MP}1) & \quad
\min_{f, z, x, y} \quad  \frac{\sum_i f_i u_i}{\sum_i f_i} \label{milpeq1}\\
    s.t. \quad & f_i = e^{-2W} + \sum_{l} \gamma_{l} (\epsilon - z_{il}), \quad\forall i \in N \label{milpeq2}\\
    & \sum_{k \in M} w_k x_{ik} - W = -\sum_{l} z_{il}, \quad\forall i \in N \label{milpeq3}\\
    & \epsilon y_{il} \leq z_{il}, z_{i(l+1)} \leq \epsilon y_{il}, \quad\forall l, \forall i \in N \label{milpeq4}\\
    & z_{il} \in [0,\epsilon], y_{il} \in \{0,1\}, \quad\forall l, \forall i \in N \label{milpeq5}\\
    \nonumber & \text{Constraints~\eqref{geneq2}-\eqref{geneq4}}
\end{align}

We can now establish the following bound.

\begin{theorem} \label{thm:approx}
Given $\epsilon < 1$, the MILP is a $2\epsilon^2$-approximation to the original problem.
\end{theorem}

\begin{proof}
To analyze the approximation bound of this MILP, we first need to analyze the tightness of the linear approximation.
Consider two points $s_1, s_2$ where $s_2 - s_1 = \epsilon$. The line segment is $t(s) = \frac{1}{\epsilon}(e^{s_2} - e^{s_1}) s - \frac{1}{\epsilon}(e^{s_2} - e^{s_1}) s_1 + e^{s_1}$. Let $\Delta(s)$ be the ratio between the line and $e^s$ on the interval $[s_1, s_2]$. It is easy to find that $\Delta(s)$ is maximized at
\[
s^* = 1 + s_1 - \frac{\epsilon}{e^{\epsilon} - 1}, \qquad \text{with} \quad    \Delta(s^*) = \frac{\frac{e^{\epsilon} - 1}{\epsilon}}{\exp\{1 - \frac{\epsilon}{e^{\epsilon} - 1}\}}.
\]
Now, let $v = \frac{e^{\epsilon} - 1}{\epsilon}$. It is known that $v \in [1,1+\epsilon]$ when $\epsilon < 1.7$. Note that $\delta(x^*) = v \exp\{\frac{1}{v}-1\} \leq 1 + (v-1)^2/2$, which holds for all $v \geq 1$. Let $\hat f(\cdot)$ be the piecewise linear approximation. For any target $i$ and observable feature configuration $x_i$, we have 
\[
\frac{\hat f(x_i)}{f(x_i)} \leq v \leq 1 + \frac{\epsilon^2}{2}.
\]

Let $x^*$ be the optimal observable features against the true score function $f$, and let $x'$ be the optimal observable features to the above MILP. Let $U(\cdot)$ be the defender's expected loss, and $\hat U(\cdot)$ be the approximate defender's expected loss. For any observable feature configuration $x$, we have
\[
\begin{split}
&|\hat U(x) - U(x)| = \left| \frac{\sum_i \hat f(x_i) u_i}{\sum_i \hat f(x_i)} - \frac{\sum_i f(x_i) u_i}{\sum_i f(x_i)} \right|\\
&= \left| \frac{\sum_i \hat f(x_i) u_i}{\sum_i \hat f(x_i)} - \frac{\sum_i \hat f(x_i) u_i}{\sum_i f(x_i)} + \frac{\sum_i \hat f(x_i) u_i}{\sum_i f(x_i)} - \frac{\sum_i f(x_i) u_i}{\sum_i f(x_i)} \right|\\
&\leq \frac{2}{\sum_i f(x_i)} \left| \sum_i f(x_i) - \sum_i \hat f(x_i) \right| = 2 \left(\frac{\sum_i \hat f(x_i)}{\sum_i f(x_i)} -1\right) \leq \epsilon^2 
\end{split}
\]
Therefore, we obtain
\[
\begin{split}
U(x') - U(x^*) &= U(x') - \hat U(x') + \hat U(x') - U(x^*)\\
&\leq U(x') - \hat U(x') + \hat U(x^*) - U(x^*)\leq 2 \epsilon^2 \qed
\end{split}
\] 
\end{proof}

\begin{algorithm}
	Initialize $L = -1, U = 1, \delta = 0, \epsilon_{bs}$\\
	\While{$U - L > \epsilon_{bs}$}{
		Solve the MILP $\mathcal{MP}1$ with objective in Eq.~\eqref{eqn:bs_obj}.\\
        \If{objective value $< 0$}{
            Let $U = \delta$
        }\Else{
            Let $L = \delta$
        }
	}
	\Return $U$, the MILP solution when $U$ was last updated.
	\caption{\textsc{Milp-bs}}\label{alg:bs}
\end{algorithm}

While $\mathcal{MP}1$ could be transformed into a MILP, the necessary linearization introduces many additional variables, increasing the size of the problem. To improve scalability, we perform binary search on the objective value $\delta$. Specifically, the objective at each iteration of the binary search becomes
\begin{equation}
    \min_{f, z, x, y} \quad \sum_i f_i u_i - \delta \sum_i f_i  \\ \label{eqn:bs_obj}.
\end{equation}
At each iteration, if the objective value of Eq.~\eqref{eqn:bs_obj} is negative, we update the binary search upper bound, and update the lower bound if positive. We proceed to the next iteration until the gap between the bounds is smaller than tolerance $\epsilon_{bs}$ and then we output the solution $x^{bs}$ when the upper bound was last updated. The complete procedure is given as Alg.~\ref{alg:bs}. Since Eq.~\eqref{eqn:bs_obj} is linear itself, we no longer need to perform linearization on it to obtain a MILP. This leads to significant speedup as we show later. We also preserve the approximation bound using triangle inequalities.

\begin{theorem} \label{thm:bs}
Given $\epsilon < 1$ and tolerance $\epsilon_\text{bs}$, binary search gives a $(2\epsilon^2 + \epsilon_\text{bs})$-approximation.
\end{theorem}

\begin{proof}
Suppose binary search terminates with interval of length $U-L \leq \epsilon_{bs}$, and observable features $x^{bs}$. Both $x^{bs}$ and the optimal observable features $x'$ to the MILP lie in this interval. This means $U(x^{bs}, \tilde f) - U(x', \tilde f) \leq \epsilon_{bs}$. Recall that $x^*$ is the optimal observable features against the true score function $f$. Therefore, we have
\[
\begin{split}
&U(x^{bs}, f) - U(x^*, f) = U(x^{bs}, f) - U(x^{bs}, \tilde f) + U(x^{bs}, \tilde f) - U(x^*, f)\\
&\leq U(x^{bs}, f) - U(x^{bs}, \tilde f) + U(x', \tilde f) + \epsilon_{bs} - U(x^*, f)\\
&\leq U(x^{bs}, f) - U(x^{bs}, \tilde f) + U(x^*, \tilde f) + \epsilon_{bs} - U(x^*, f)\\
&\leq 2\epsilon^2 + \epsilon_{bs} \qed
\end{split}
\]

\end{proof}

Now, we connect the learning and planning results together. Suppose we learned an approximate score function $\hat f$ (Theorem~\ref{thm:samplecomplexity}), and we find an approximately optimal feature configuration (Theorem~\ref{thm:approx}) assuming $\hat f$. The following result shows that we can still guarantee end-to-end approximate optimality.

\begin{theorem}
\label{thm:nearopt}
Suppose for some $\epsilon \leq 1/4$, $\frac{1}{1+\epsilon} < \frac{\hat f(x_i)}{f(x_i)} < 1 + \epsilon$ for all $x_i$.
Then, $|U(x, \hat f) - U(x, f)| \leq 4\epsilon$ for all $x$. Let $x^* = \arg\min_x U(x, f)$ and $x''$ be such that $U(x'', \hat f) \leq \min_x U(x, \hat f) + \eta$, then $U(x'', f) - U(x^*, f) \leq 8\epsilon + \eta$.
\end{theorem}

\begin{proof}
Let $\hat f(x_i) = \exp(\sum_k \hat w_k x_{ik})$ and $f(x_i) = \exp(\sum_k w_k x_{ik})$. Since \[\frac{1}{1+\epsilon} < \frac{\hat f(x_i)}{f(x_i)} < 1 + \epsilon,\]
we get 
\[
\begin{split}
&-\epsilon \leq-\ln(1+\epsilon)<\sum_k (\hat w_k - w_k) x_{ik}= \ln \frac{\hat f(x_i)}{f(x_i)} < \ln(1+\epsilon) \leq \epsilon.
\end{split}
\]
That is, $|\sum_k (\hat w_k - w_k) x_{ik}| < \epsilon$. The proof of Theorem 3.7 in~\cite{haghtalab2016IJCAI} now follows to prove the first part of Theorem~\ref{thm:nearopt} if we redefine their $u_i(p_i)$ as $\sum_{k \in M} w_k x_{ik}$ and $\hat u_i(p_i)$ as $\sum_{k \in M} \hat w_k x_{ik}$. For completeness, we adapt their proof below using our notations. 

As defined in Section~\ref{sec:samplecomplexity}, $D^x(t) = \frac{f(x_t)}{\sum_i  f(x_i)}$ and $\hat D^x(t) = \frac{\hat f(x_t)}{\sum_i \hat f(x_i)}$. We have
\[
\begin{split}
&\left|\ln \frac{\hat D^x(t)}{D^x(t)}\right| = \left| \left(\sum_k (\hat w_k - w_k) x_{tk} \right) - \ln \frac{\sum_i \exp\{\sum_k \hat w_k  x_{ik} \}}{\sum_i \exp\{\sum_k w_k  x_{ik} \}} \right|\\
&\leq \left| \sum_k (\hat w_k - w_k) x_{tk} \right| + \left| \ln \frac{\sum_i \exp\{\sum_k w_k  x_{ik} \} \exp\{\sum_k (\hat w_k - w_k)  x_{ik} \}}{\sum_i \exp\{\sum_k w_k  x_{ik} \}} \right|\\
&< \epsilon + \max_i \left| \ln \exp\{\sum_k (\hat w_k - w_k)  x_{ik} \} \right|< 2\epsilon
\end{split}
\]
Using a few inequalities we can bound $\left| \frac{\hat D^x(t)}{D^x(t)} - 1 \right| \leq 4\epsilon$. This leads to, for all $x$,
\[
\begin{split}
&|U(x, \hat f) - U(x, f)| = \left| \sum_{i\in N} (\hat D^x(i) - D^x(i)) u_i \right|\leq \sum_{i\in N} \left| \hat D^x(i) - D^x(i)\right| \left| u_i \right|\\
&= \sum_{i\in N} \left| \frac{\hat D^x(i)}{D^x(i)} - 1\right| \left| u_i \right| D^x(i) \leq 4\epsilon \sum_{i\in N} \left| u_i \right| D^x(i) \leq 4\epsilon \max_{i\in N} \left| u_i \right| \leq 4\epsilon
\end{split}
\]
Let $x^* = \arg\min_x U(x, f)$ be the true optimal feature configuration, $x'= \arg\min_x U(x, \hat f)$ be the optimal configuration using the learned score function $\hat f$, and $x''$ be an approximate optimal configuration against $\hat f$, i.e., $U(x'', \hat f) \leq U(x', \hat f) + \eta$. 
We have 
\[
\begin{split}
&U(x'', f) \leq U(x'', \hat f) + 4\epsilon \leq U(x’, \hat f) + 4\epsilon + \eta \leq U(x^*, \hat f) + 4\epsilon + \eta \leq U(x^*, f) + 8\epsilon + \eta. \qed
\end{split}
\]
\end{proof}

In addition, we propose two exact algorithms for special cases of FDP, which can be found in Appendix~\ref{app:specialcase}. When the deception cost is associated with discrete features only, we provide an exact MILP formulation. When there is no budget and feasibility constraints, we can find the optimal defender strategy in $O(n \log n + m)$ time using a greedy algorithm. Inspired by this greedy algorithm, we introduce a greedy heuristic for the general case.
\textsc{Greedy} (Alg.2 in Appendix~\ref{app:algs}) finds the feature vectors that maximize and minimize the score, respectively, using gradient descent-based algorithm. It then greedily applies these features to targets of extreme losses. We show its performance in the following section as well.

\section{Experiments}
\label{sec:experiments}
We present the experimental results for our learning and planning algorithms separately, and then combine them to demonstrate the effectiveness of our learning and planning framework. All experiments are carried out on a 3.8GHz Intel Core i5 CPU with 32GB RAM. We use Ipopt as our non-convex solver and CPLEX 12.8 as the MILP solver. All results are averaged over 20 instances; error bars represent standard deviations.
Details about hyper-parameters can be found in Appendix~\ref{app:param}.

\subsection{Learning}
\paragraph{Classical score function} First, we assume the adversary uses the classical score function in Eq~\eqref{eqn:score}. The defender learns this score function using the closed-form estimation (CF) in Theorem~\ref{thm:samplecomplexity}.
We study how the learning accuracy changes with the size of training sample $d$. 
We sample the parameters of the true score function $f$ uniformly at random from $[-0.5, 0.5]$.
We then generate $m$ feature configurations uniformly at random.
For each of them, we sample the attacked target $d/m$ times according to $f$, 
obtaining a training set of $d$ samples. We generate a test set $\tilde D$ of $5\times 10^5$ configurations sampled uniformly at random. We measure the learning error as the mean total variation distance between the attack distribution from the learned $\hat f$ and that of the true model $f$:

\[
\frac{1}{|\tilde D|} \sum_{j=1}^{|\tilde D|} d_{TV}\left(\left(\frac{f(x^j_i)}{\sum_{t \in N} f(x^j_t)}\right)_{i\in N}, \left(\frac{\hat f(x^j_i)}{\sum_{t \in N} \hat f(x^j_t)}\right)_{i\in N}\right).
\]
Figure~\ref{fig:learn_simple} shows that the learning error decreases as we increase the number of samples. Theorem~\ref{thm:samplecomplexity} provides a sample complexity bound, which we annotate in Figure~\ref{fig:learn_simple} as well. The experiment shows that we need much fewer samples to learn a relatively good score function, and smaller games exhibit smaller learning error.

\begin{figure*}[t]
     \subfloat[Learning classical score function in Eq.~\eqref{eqn:score}]{\includegraphics[width=0.24\columnwidth]{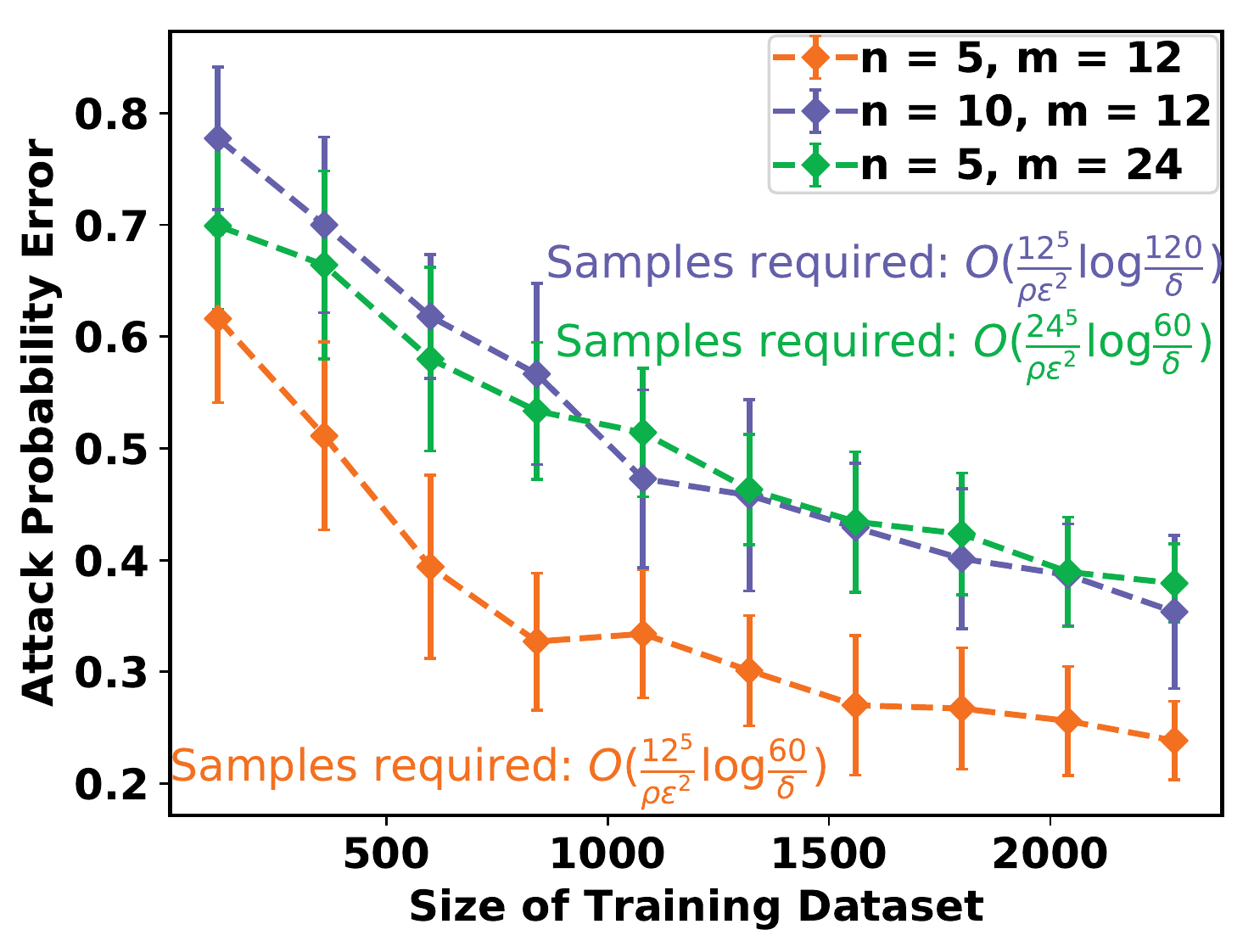}
        \label{fig:learn_simple}} \,
     \subfloat[Learning NN-3 score function]{\includegraphics[width=0.24\columnwidth]{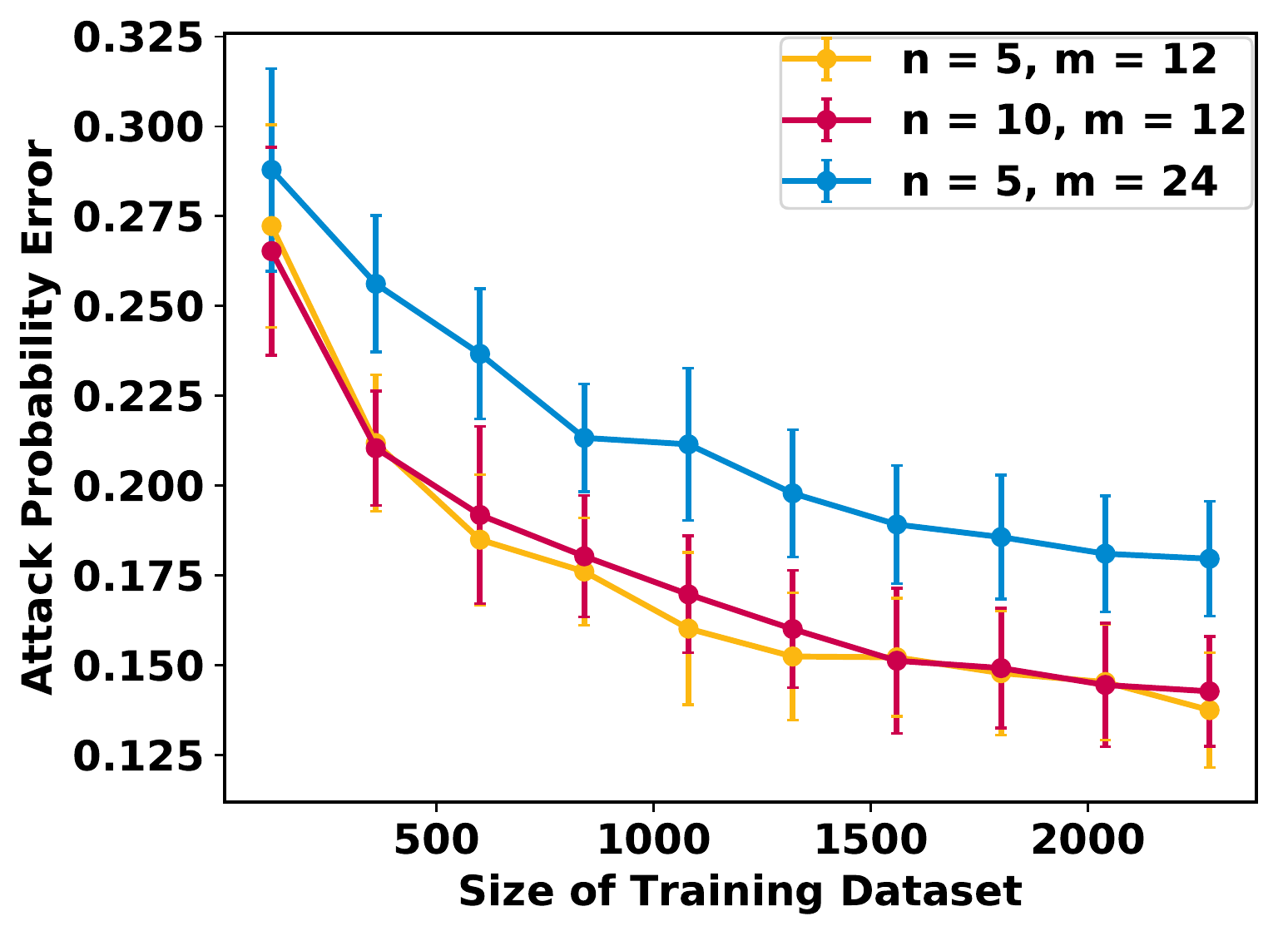}%
        \label{fig:learn_complex}} \,
    \subfloat[Planning with classical score function in Eq.~\eqref{eqn:score}, $m=12$]{\includegraphics[width=0.24\columnwidth]{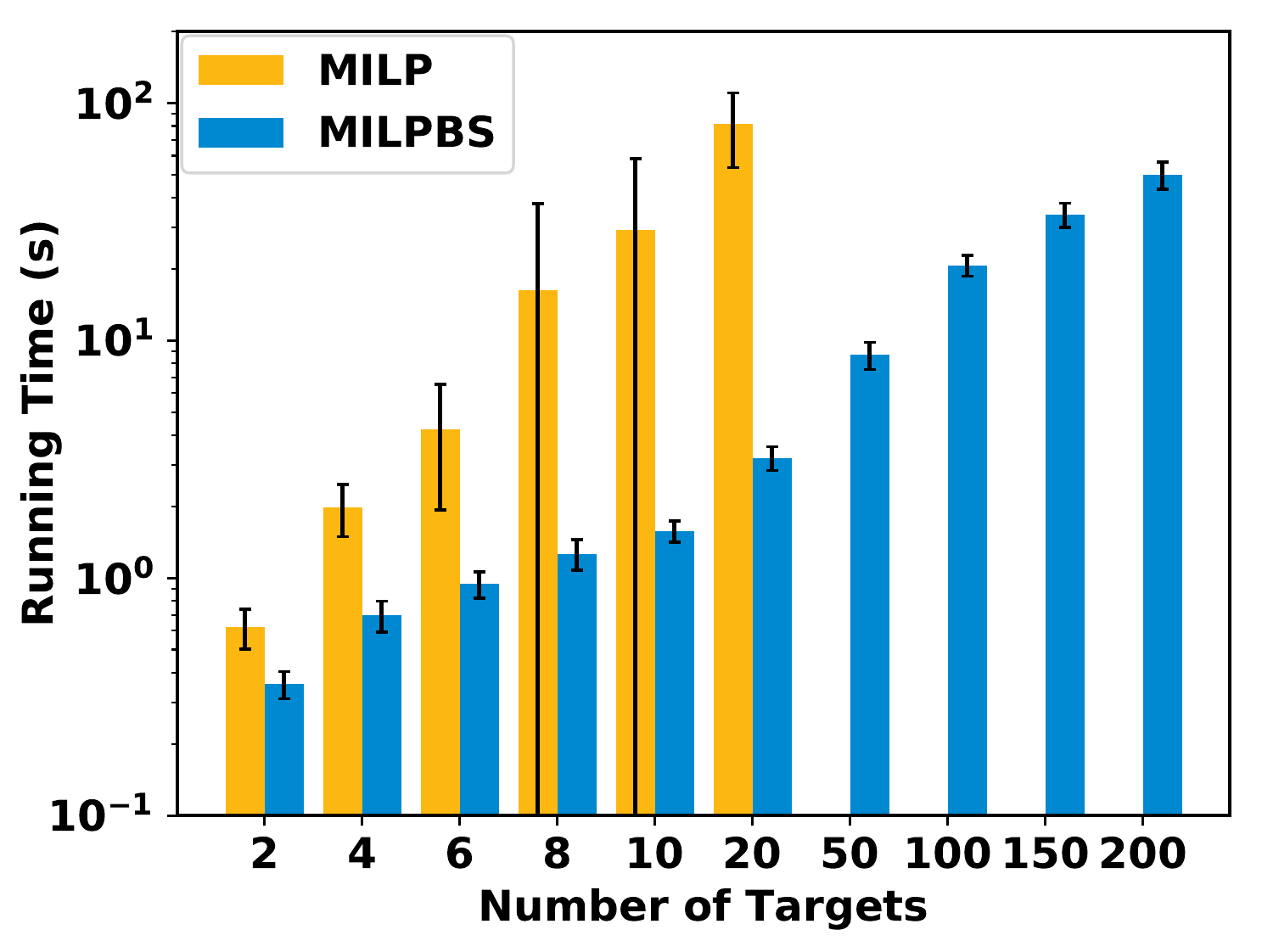}%
        \label{fig:plan_simple_general_N}} \,
     \subfloat[Planning with NN-3 score function, $m=12$]{\includegraphics[width=0.24\columnwidth]{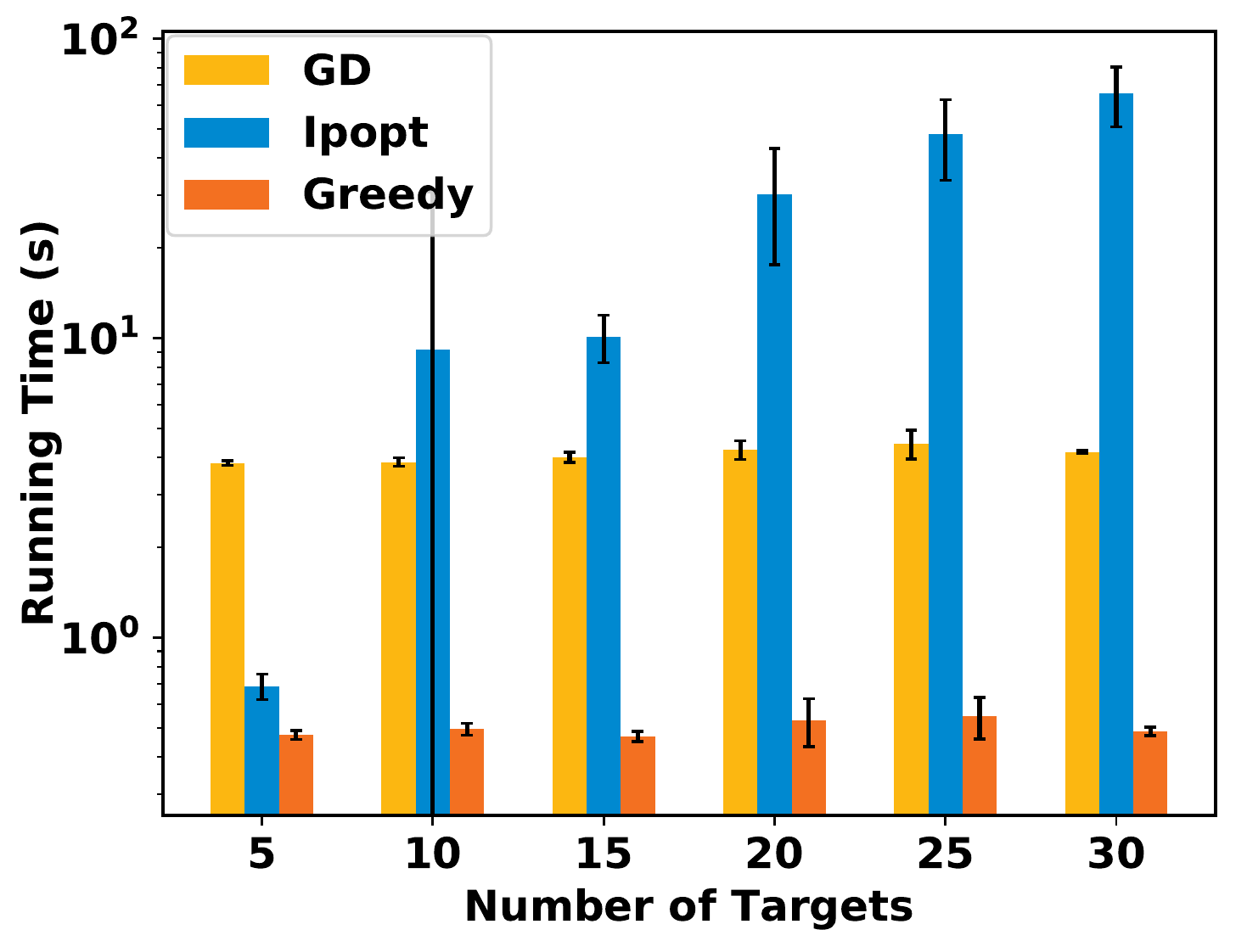}%
        \label{fig:plan_complex_runtime_N}} \\
     \subfloat[Planning with NN-3 score function, $m=12$]{\includegraphics[width=0.23\columnwidth]{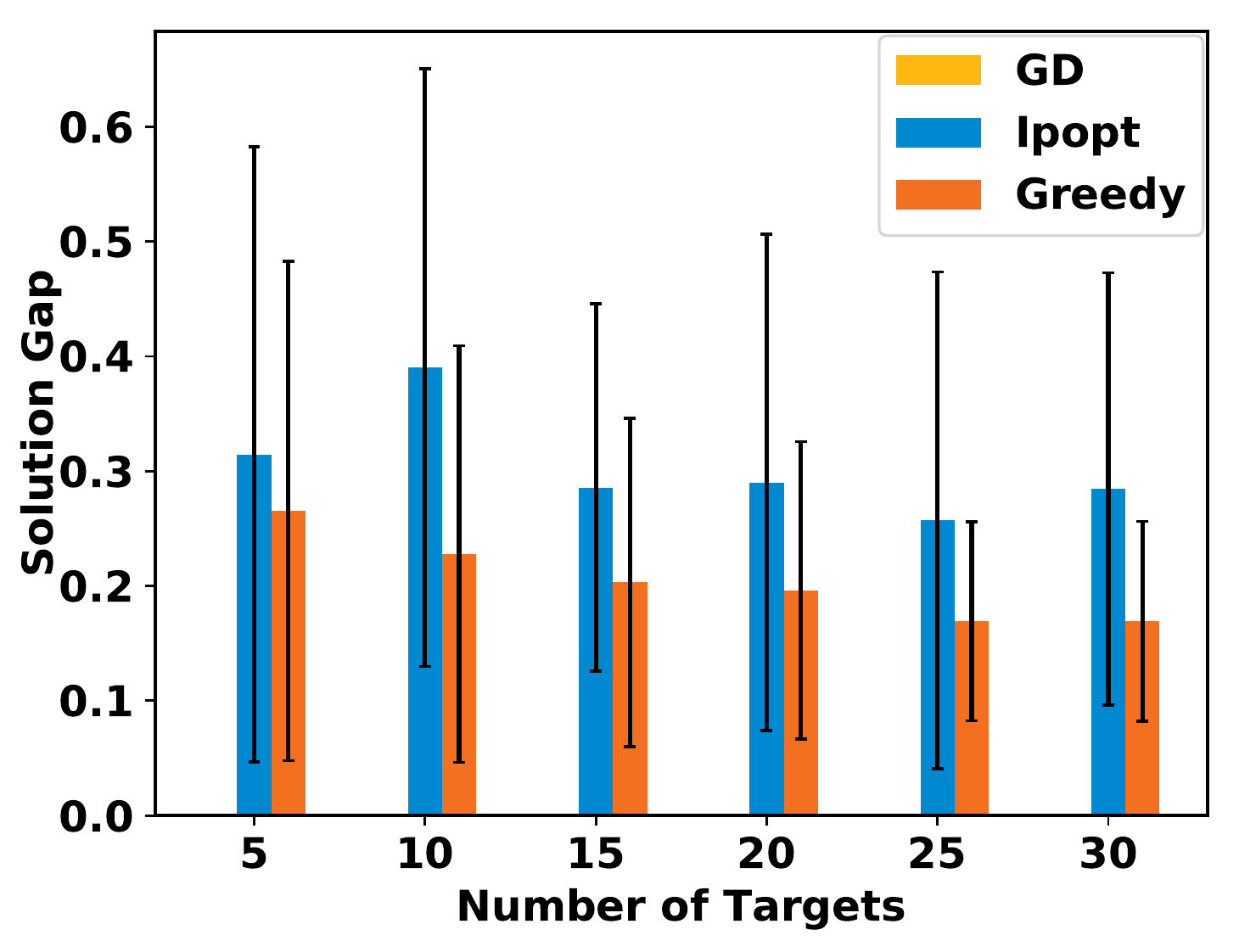}
        \label{fig:plan_complex_error_N}} \,
    \subfloat[Learning + planning, classical score function, $n=5, m=12$]{\includegraphics[width=0.24\columnwidth]{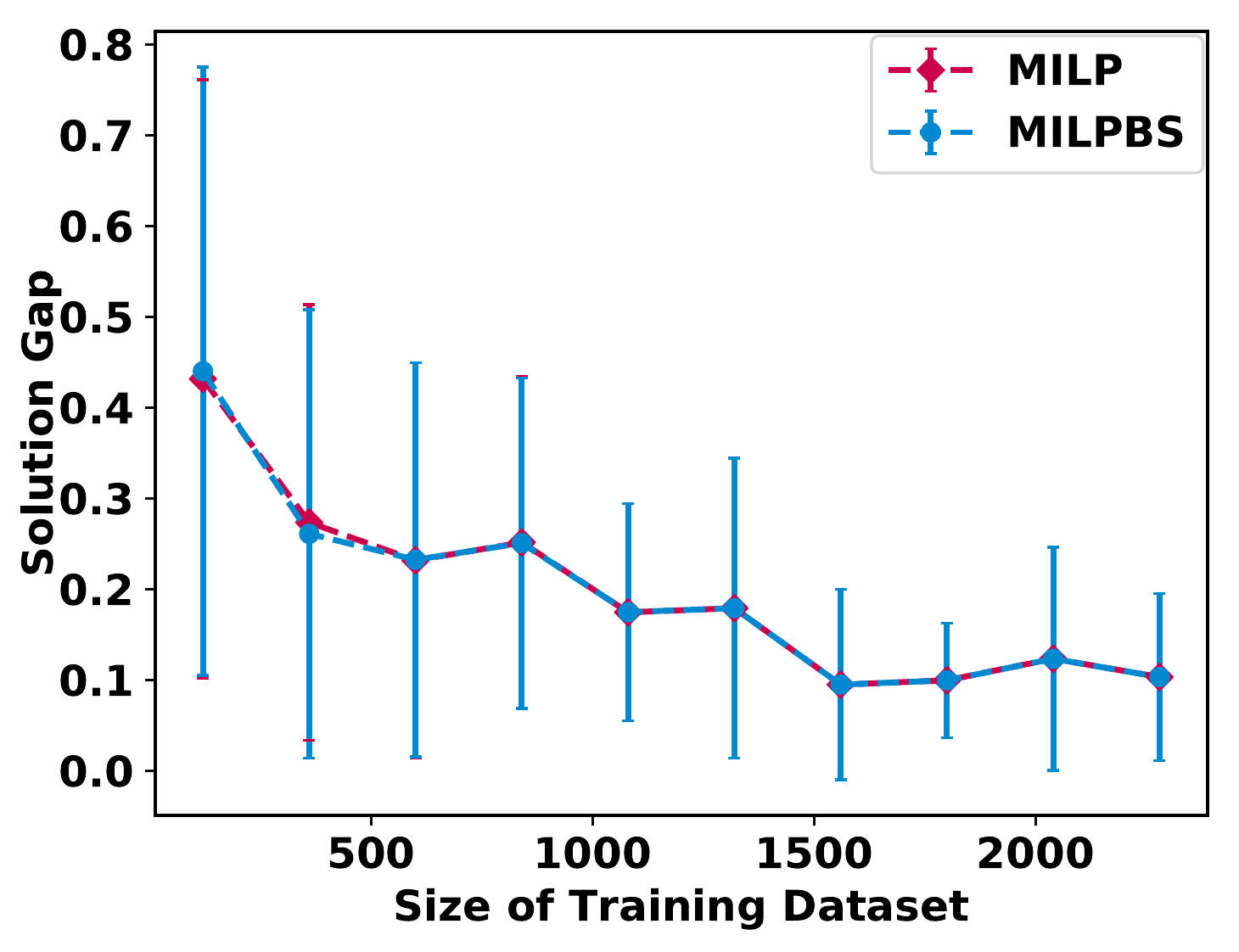}%
        \label{fig:learn_plan_simple_D}}\,
     \subfloat[Learning + Planning, classical score function, $m=12$]{\includegraphics[width=0.24\columnwidth]{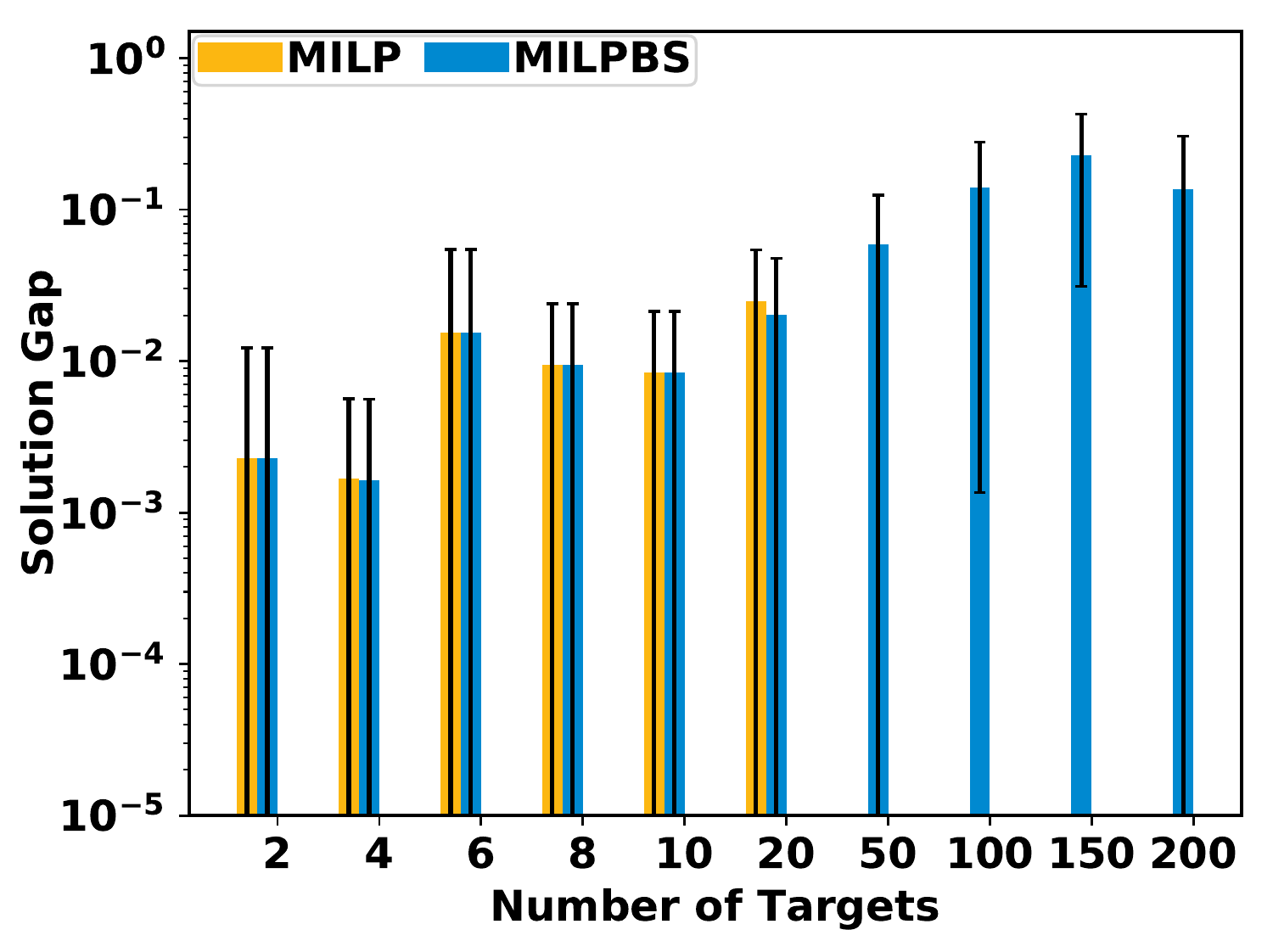}
        \label{fig:learn_plan_simple_N}}\,
    \subfloat[Learning + Planning, NN-3 score function, $n=5, m=12$]{\includegraphics[width=0.24\columnwidth]{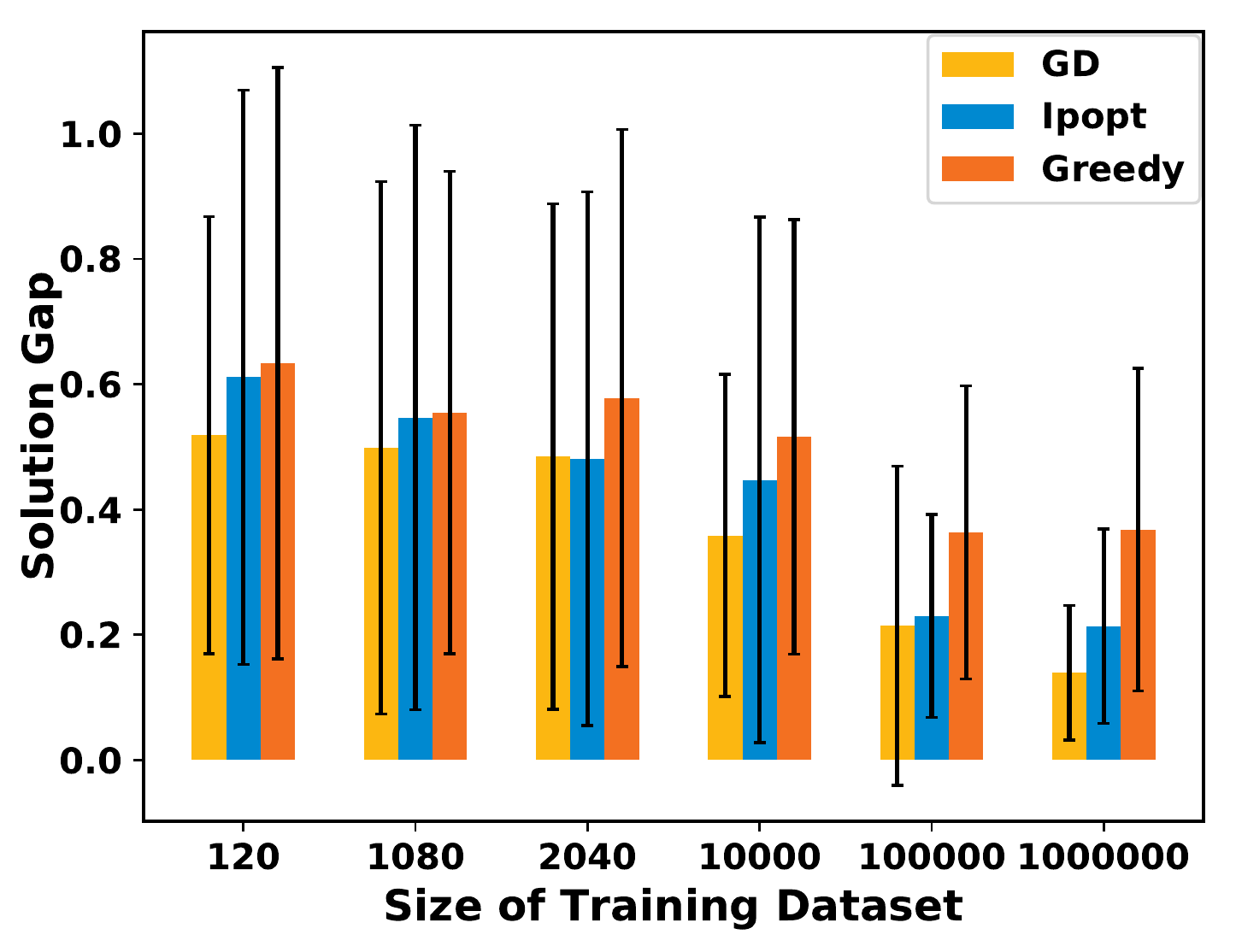}
        \label{fig:learn_plan_complex_D}}
    \caption{Experimental results}
    \label{fig:all}
\end{figure*}

\paragraph{3-layer NN represented (NN-3) score function} We assume the adversary uses a 3-layer neural network score function,
whose details are in Appendix~\ref{app:param}.
We use the gradient descent-based (GD) learning algorithm RMSProp as described in Section~\ref{sec:samplecomplexity}, with learning rate 0.1. For each sample size $d$, we generate $d$ feature configurations and sample an attacked target for each of them in the training set. 
Fig.~\ref{fig:learn_complex} shows GD can minimize the learning error to below 0.15. Note that the training data are different in Fig.~\ref{fig:learn_simple} and~\ref{fig:learn_complex}, thus the two figures are not directly comparable.

We also measured $|\hat \theta - \theta|$, the $L_1$ error in the score function parameter $\theta$, which directly relates to the sample complexity bound in Theorem~\ref{thm:samplecomplexity}. We include the results in Appendix~\ref{app:exp}.

\subsection{Planning}
We test our algorithms on finding the optimal feature configuration against a known attacker model. 
The FDP parameter distributions are included in Appendix~\ref{app:param}.

\paragraph{Classical score function}
Fig.~\ref{fig:plan_simple_general_N} shows that the binary search version of the MILP based on $\mathcal{MP}1$ (MILPBS) runs faster than that without binary search on most instances. MILPBS scales up to problems with 200 targets, which is already at the scale of many real-world problems. MILP does not scale beyond problems with 20 targets.
In Appendix~\ref{app:exp}, we show that MILPBS also scales better in terms of the number of features.
We set the MILP's error bound at 0.005 and $\epsilon_\text{bs}=1e-4$; the difference in the two algorithms' results is negligible.

\paragraph{NN-3 score function} When the features are continuous without feasibility constraints, planning becomes a non-convex optimization problem. We can apply the gradient-based optimizer or non-convex solver.
Recall that $U(x)$ is the defender's expected loss using feature configuration $x$. We measure the solution gap of alg $\in $ $\{\text{Ipopt, GD, \textsc{Greedy}}\}$ as $\frac{U(x^{\text{alg}}) - U(x^{\text{GD}})}{U(x^{\text{GD}})}$, where $x^{\text{alg}}$ is the solution from the corresponding algorithm.

Fig.~\ref{fig:plan_complex_runtime_N} and \ref{fig:plan_complex_error_N} show the running time and solution gap fixing $m= 12$. The running time of GD and \textsc{Greedy} does not change much across different problem sizes, yet Ipopt runs slower than the former two on most problem instances.
GD also has smaller solution gap than Ipopt and \textsc{Greedy}.
In Appendix~\ref{app:exp} we show the number of features affect these metrics in a similar way.

\subsection{Combining Learning and Planning}
We integrate the learning and planning algorithms to examine our full framework. The defender learns a score function $\hat f$ using algorithm $\text{L}$. Then, she uses planning algorithm $\text{P}$ to find an optimal configuration $x^{\text{L}, \text{P}}$ assuming $\hat f$. 
We measure the solution gap as $\frac{U(x^{\text{L}, \text{P}}) - U(x^*)}{U(x^*)}$, where $x^*$ is the optimal feature configuration against the true attacker model, computed using MILPBS or GD.

\paragraph{Classical score function}
We test learning algorithm $\text{CF}$ and planning algorithms $\text{P} \in \{\text{MILP, MILPBS}\}$. 
Fig.~\ref{fig:learn_plan_simple_D} shows how the solution gap changes with the size of the training dataset.
With $n \leq 20$ targets, all algorithms yield solution gaps below 0.1 (Fig.~\ref{fig:learn_plan_simple_N}).
The reader might note the overlapping error bars, which are expected since MILP and MILPBS should not differ much in solution quality.
Indeed, the difference is negligible as the smallest p-value of the 6 paired t-tests (fixing the number of targets for which they are tested) is 0.16.

\paragraph{NN-3 score function}
We test learning algorithm GD and planning algorithms $\text{P} \in \{\text{GD, Ipopt, \textsc{Greedy}}\}$. 
Fig.~\ref{fig:learn_plan_complex_D} shows how the solution gap changes with the size of training dataset $d$.
Paired t-tests suggest that GD has significantly smaller solution gap than \textsc{Greedy} ($p < 0.03$) at each size of training dataset except 1080.
Ipopt also has significantly smaller solution gap than \textsc{Greedy} ($p < 0.01$) when on large datasets with $d \geq 10^5$ samples. On the largest dataset $d = 10^6$, GD also performs significantly better than Ipopt ($p = 0.04$).

Compared to the case with classical score functions, more data are required here to achieve a small solution gap. 
Since learning error is small for both cases (Fig.~\ref{fig:learn_simple},\ref{fig:learn_complex}), this suggests planning is more sensitive to NN-3 score functions than classical score functions.

\subsection{Case Study: Credit Bureau Network}
\label{sec:casestudy}
The financial sector is a major victim of cyber attacks due to its large amount of valuable information and relatively low level of security measures. In this case study, we ground our FDP model in a credit bureau's network. We show how feature deception improves the network security when the attacker follows a domain-specific rule-based behavioral model.

We note that the purpose of this case study is not to show the scalability of our algorithm: all previous experiments fulfill that purpose. Instead, here we demonstrate why deception is useful, how our algorithm yields deception strategies reasonable in the real world, and how our algorithm capably handles an attacker which does not conform to our assumed score function.

As shown in Table~\ref{tab:casestudy}, we consider a network of 10 nodes (i.e. targets) with 6 binary features: operating system (Windows/Linux) and the availability of SMTP, NetBIOS, HTTP, SQL, and Samba services. Each node has a type of server running on it, which determines the features available on that node. Some nodes would incur a high loss if attacked, like the database servers, because for a credit bureau the safety of users' credit information is of utmost importance. Others might incur a low loss, such as the mail servers and the web server. Nodes of the same type might lead to different losses. For example, some database servers might have access to more information than others.
Each feature has different switching cost $c_k$. For the operating system, the cost is $c_k = 5$. For SQL, Samba, and HTTP services, the cost is 2. The cost is 1 for others. The defender has a budget of 10. There is no constraint on switching each individual feature, i.e. $C(\hat x_{ik}) = \{0, 1\}$. However, we impose that Windows + Samba and Linux + NetBIOS cannot be present on the same node, as it is technically impossible to do so.

\begin{table}[t] 
\centering
	\begin{tabular}{c c c c}
		\hline \textbf{Node type} & \textbf{Node ID} & \textbf{Actual features $\hat x_i$} & \textbf{Loss $u_i$} \\ \hline
		Mail server & 0, 1 & Windows, SMTP, NetBIOS & 0.1\\ 
		Web server & 2 & Windows, HTTP & 0.2\\
		App server & 3, 4 & Windows, SQL, NetBIOS & 0.3\\ 
		Database server & 5,6,7  & Linux, SQL, SMTP, Samba & 0.4\\ 
		Database server & 8,9  & Linux, SQL, SMTP, Samba & 0.8\\ \hline
	\end{tabular}
	\caption{Feature configuration of a typical credit bureau computer network.}
	\label{tab:casestudy} 
\end{table}

We demonstrate the entire learning and planning pipeline. 
We use an attacker's behavior model common in the security analysis. The attacker cares about a subset $M'\subseteq M$ of the features, and we call each such feature $k \in M'$ a requirement. The attack is uniformly randomized among the targets that satisfy the most number of requirements. Although this decision rule does not fit our classical score functions, we can approximate it by giving large weights $w_k$ to the requirement features, and 0 to the rest.

\begin{table}[t]
\centering
	\begin{tabular}{cccc}
		\cline{1-4}
		\textbf{Attacker} & \textbf{Solution $x_i$} & \textbf{Attacked nodes} & \textbf{Loss}\\
		\cline{1-4}
		APT & \thead{\normalsize Node 1: Windows $\rightarrow$ Linux \\ \normalsize  Node 1: SQL off $\rightarrow$ on  \\ \normalsize  Node 1: NetBIOS on $\rightarrow$ off   \\ \normalsize  Node 8, 9: SMTP on $\rightarrow$ off} &  \thead{\normalsize5,6,7,8,9  \\\normalsize$\rightarrow$1 ,5, 6, 7} & \thead{\normalsize $0.56 \rightarrow$ \\\normalsize$0.325$}\\
		Botnet  & \thead{\normalsize  Node 3: NetBIOS on $\rightarrow$ off   \\ \normalsize Node 4: NetBIOS on $\rightarrow$ off}  &  \thead{\normalsize0,1,3,4  \\\normalsize$\rightarrow$0,1} & \thead{\normalsize $0.2 \rightarrow$ \\\normalsize$0.1$}\\
		\hline
	\end{tabular}
	\caption{Learning + planning results for 2 types of attackers.} \label{tab:casestudyresults}
\end{table}

First, we consider an APT-like attacker, who wants to exfiltrate data by exploiting the SMTP service. They have expertise in Linux systems and want to maintain a high degree of stealth. Thus, their decision rule is based on the three requirement features: Linux, SMTP, and SQL.
Without deception, the attacker would randomize attack over nodes 5-9, because these nodes satisfy 3 requirements and other nodes satisfy at most 2.
As shown in Table~\ref{tab:casestudyresults}, the optimal solution for the learning and planning problem leads to an expected defender's loss of 0.325, which is a 42\% decrease from the loss with no deception. With limited budget, the defender makes the least harmful target, node 1, very attractive and the most harmful targets, nodes 8 and 9, less attractive.

We also consider a botnet attacker, who wants to create a bot by exploiting the NetBIOS service. They have expertise in Windows and want to maintain a moderate degree of stealth. Thus, their decision rule is based on two requirement features: Windows and NetBIOS. The results in Table~\ref{tab:casestudyresults} shows that the defender should set the NetBIOS observed value to be off for nodes 3 and 4, attracting the attacker to the least harmful nodes. This decreases the defender's expected loss by 50\% compared to not using deception.

\section{Related Work}

\paragraph{Deception}
Deception has been studied in many domains, and of immediate relevance is its use in cybersecurity~\cite{rowe2007CWCT}. Studies have suggested that deceptively responding to an attacker's scanning and probing could be a useful defensive measure~\cite{jajodia2017IEEE,albanese2016CD}. 
Schlenker et al.~\cite{schlenker2018AAMAS} and Wang and Zeng~\cite{wang2018Gamesec} propose game-theoretic models where the defender manipulates the query response to a known attacker. Proposing a domain-independent model, we advance the state of the art by (1) providing a unified learning and planning framework with theoretical guarantee which can deal with unknown attackers, (2) extending the finite ``type'' space in both papers, where ``type'' is defined by the combination of feature values, to an infinite feature space that allows for both continuous and discrete features, and (3) incorporating a highly expressive bounded rationality model whereas both papers assume perfectly rational attackers.

For the more general case, Horak et al.~\cite{horak2017Gamesec} study a defender that engages an attacker in a sequential interaction. A complementary view where the attacker aims at deceiving the defender is provided in~\cite{nguyen2019AAAI,gan2019imitative}. Different from them, we assume no knowledge of the set of possible attacker types.
In \cite{yin2014AAAI,guo2017IJCAI,nguyen2019AAAI,gan2019imitative} deception is defined as deceptively allocating defensive resources. We study feature deception where no effective tools can thwart an attack, which is arguably more realistic in high-stakes interactions. 
When such tools exist, feature deception is still valuable 
for strategic defense.

\paragraph{Learning in Stackelberg games}
Much work has been devoted to learning in Stackelberg games. Our work is most directly related to that of Haghtalab et al.~\cite{haghtalab2016IJCAI}. They show that three defender strategies are sufficient to learn a SUQR-like adversary behavior model in Stackelberg security games. The only decision variable in their model, the coverage probability, may be viewed as a single feature in FDP. FDP allows for an arbitrary number of features, and this realistic extension makes their key technique inapplicable for analyzing the sample complexity. Our main learning result also removes the technical constraints on defender strategies present in their work. Sinha et al.~\cite{sinha2016AAMAS} study learning adversary's preferences in a probably approximately correct (PAC) setting. However, their learning accuracy depends heavily on the quality of distribution from which they sample the defender's strategies. We provide a uniform guarantee in a distribution-free context. Other papers~\cite{blum2014NIPS,marecki2012AAMAS,letchford2009SAGT,peng2019AAAI} study the online learning setting with rational attackers. As pointed out in~\cite{haghtalab2016IJCAI}, considering the more realistic bounded rationality scenario allows us to make use of historical data and use our algorithm more easily in practice.

\paragraph{Planning with boundedly rational attackers}
Yang et al.~\cite{yang2012AAMAS} propose a MILP-based solution in security games. Our planning algorithm goes beyond the coverage probability and determines the configuration of multiple features, and adopt a more expressive behavior model. The subsequent papers that incorporate learning with such bounded rationality models do not provide any theoretical guarantee~\cite{yang2014AAMAS,fang2015IJCAI}. 
A recent work develops a learning and planning pipeline in security games~\cite{perrault2019decision}. However, their algorithm requires the defender know a priori some parameters in the attacker's behavior model, and provides no global optimality guarantee.

\section{Discussion}
\label{sec:discussion}
We conclude with a few remarks regarding the generality and limitations of our work. 
First, our model allows the attacker to have knowledge of deception if the knowledge is built into their behavior. For example, the attacker avoids attacking a target because it is ``too good to be true''. This can be captured by a score function that assigns a low score for such a target.

Second, our model can handle sophisticated attackers who can outstrip deception.
A singleton feasible set $C(\hat x_{ik})$ implies the defender knows the attacker can find out the actual value of a feature.
As an important next step, we will study the change of attacker’s belief of deception over repeated interactions.

Third, typically, actual features on functional targets are environmental parameters beyond the defender's control, or at least have high cost of manipulation.
Altering them and defender's losses $u_i$ does not align conceptually with deception. Thus, we treat them as fixed. For a target with no fixed actual values, e.g., a honeypot, the defender's cost is just the cost of configuring the feature, e.g., installing Windows. For consistency, we can set $\hat x_{ik}$ as the feature value with the lowest configuration cost, and $\eta_{ik}$ is the additional cost for a different feature value.

Fourth, the attacker’s preference might shift when there is a major change in security landscape, e.g. a new vulnerability disclosed. In such case, a proactive defender will recalibrate the system: recompute the attacker’s model and reconfigure the features. Moreover, exactly because the defender has learned the preferences before the change using our algorithms, the defender now knows better what qualifies as a major change. Our algorithms are fast enough for a proactive defender to run regularly.

Fifth, when faced with a group of attackers, in FDP we learn an average behavioral model of the population. To handle multiple attacker types, one could refer to the literature on Bayesian Stackelberg games~\cite{paruchuri2008AAMAS}.
 
Finally, in FDP the defender uses only pure strategies. In many domains such as cybersecurity, frequent system reconfiguration is often too costly. Thus, the system appears static to the attacker. We leave to future work to explore mixed strategies in applications where they are appropriate.


\section*{Acknowledgments}
This research was sponsored by the Combat Capabilities Development Command Army Research Laboratory and was accomplished under Cooperative Agreement Number W911NF-13-2-0045 (ARL Cyber Security CRA). The views and conclusions contained in this document are those of the authors and should not be interpreted as representing the official policies, either expressed or implied, of the Combat Capabilities Development Command Army Research Laboratory or the U.S. Government. The U.S. Government is authorized to reproduce and distribute reprints for Government purposes not withstanding any copyright notation here on.

\bibliographystyle{splncs04}
\bibliography{ijcai19}

\begin{thebibliography}{10}
\providecommand{\url}[1]{\texttt{#1}}
\providecommand{\urlprefix}{URL }
\providecommand{\doi}[1]{https://doi.org/#1}

\bibitem{abbasi2016know}
Abbasi, Y., Kar, D., Sintov, N., Tambe, M., Ben-Asher, N., Morrison, D.,
  Gonzalez, C.: Know your adversary: Insights for a better adversarial
  behavioral model. In: CogSci (2016)

\bibitem{albanese2016CD}
Albanese, M., Battista, E., Jajodia, S.: Deceiving attackers by creating a
  virtual attack surface. In: Cyber Deception (2016)

\bibitem{blum2014NIPS}
Blum, A., Haghtalab, N., Procaccia, A.D.: Learning optimal commitment to
  overcome insecurity. In: NIPS (2014)

\bibitem{chiang2016acyds}
Chiang, C.Y.J., Gottlieb, Y.M., Sugrim, S.J., Chadha, R., Serban, C.,
  Poylisher, A., Marvel, L.M., Santos, J.: Acyds: An adaptive cyber deception
  system. In: MILCOM (2016)

\bibitem{fang2015IJCAI}
Fang, F., Stone, P., Tambe, M.: When security games go green: Designing
  defender strategies to prevent poaching and illegal fishing. In: IJCAI (2015)

\bibitem{gan2019imitative}
Gan, J., Xu, H., Guo, Q., Tran-Thanh, L., Rabinovich, Z., Wooldridge, M.:
  Imitative follower deception in stackelberg games. In: EC (2019)

\bibitem{guo2017IJCAI}
Guo, Q., An, B., Bosansk{\`y}, B., Kiekintveld, C.: Comparing strategic secrecy
  and stackelberg commitment in security games. In: IJCAI. pp. 3691--3699
  (2017)

\bibitem{haghtalab2016IJCAI}
Haghtalab, N., Fang, F., Nguyen, T.H., Sinha, A., Procaccia, A.D., Tambe, M.:
  Three strategies to success: Learning adversary models in security games. In:
  IJCAI (2016)

\bibitem{hinton2012rmsprop}
Hinton, G., Srivastava, N., Swersky, K.: Neural networks for machine learning
  lecture 6a overview of mini-batch gradient descent  (2012)

\bibitem{horak2017Gamesec}
Hor{\'a}k, K., Zhu, Q., Bo{\v{s}}ansk{\`y}, B.: Manipulating adversary’s
  belief: A dynamic game approach to deception by design for proactive network
  security. In: GameSec (2017)

\bibitem{hurlburt2016Computer}
Hurlburt, G.: ``good enough'' security: The best we'll ever have. Computer
  (2016)

\bibitem{jafarian2012openflow}
Jafarian, J.H., Al-Shaer, E., Duan, Q.: Openflow random host mutation:
  transparent moving target defense using software defined networking. In:
  Proceedings of the first workshop on Hot topics in software defined networks.
  ACM (2012)

\bibitem{jajodia2017IEEE}
Jajodia, S., Park, N., Pierazzi, F., Pugliese, A., Serra, E., Simari, G.I.,
  Subrahmanian, V.: A probabilistic logic of cyber deception. IEEE Transactions
  on Information Forensics and Security  \textbf{12}(11) (2017)

\bibitem{jajodia2016Springer}
Jajodia, S., Subrahmanian, V., Swarup, V., Wang, C.: Cyber deception. Springer
  (2016)

\bibitem{latimer2001deception}
Latimer, J.: Deception in War. John Murray (2001)

\bibitem{letchford2009SAGT}
Letchford, J., Conitzer, V., Munagala, K.: Learning and approximating the
  optimal strategy to commit to. In: SAGT (2009)

\bibitem{lyon2009nmap}
Lyon, G.F.: Nmap network scanning: The official Nmap project guide to network
  discovery and security scanning. Insecure (2009)

\bibitem{marecki2012AAMAS}
Marecki, J., Tesauro, G., Segal, R.: Playing repeated stackelberg games with
  unknown opponents. In: AAMAS (2012)

\bibitem{nakashima2013}
Nakashima, E.: To thwart hackers, firms salting their servers with fake data

\bibitem{nguyen2019AAAI}
Nguyen, T.H., Wang, Y., Sinha, A., Wellman, M.P.: Deception in finitely
  repeated security games. In: AAAI (2019)

\bibitem{paruchuri2008AAMAS}
Paruchuri, P., Pearce, J.P., Marecki, J., Tambe, M., Ordonez, F., Kraus, S.:
  Playing games for security: An efficient exact algorithm for solving bayesian
  stackelberg games. In: AAMAS (2008)

\bibitem{peng2019AAAI}
Peng, B., Shen, W., Tang, P., Zuo, S.: Learning optimal strategies to commit
  to. In: AAAI (2019)

\bibitem{perrault2019decision}
Perrault, A., Wilder, B., Ewing, E., Mate, A., Dilkina, B., Tambe, M.:
  Decision-focused learning of adversary behavior in security games. arXiv
  preprint arXiv:1903.00958  (2019)

\bibitem{potter2009CFS}
Potter, B., Day, G.: The effectiveness of anti-malware tools. Computer Fraud \&
  Security  (2009)

\bibitem{provos2004virtual}
Provos, N., et~al.: A virtual honeypot framework. In: USENIX Security Symposium
  (2004)

\bibitem{pwc2017}
PwC: Operation Cloud Hopper Technical Annex

\bibitem{rowe2007CWCT}
Rowe, N.C.: Deception in defense of computer systems from cyber attack. In:
  Cyber Warfare and Cyber Terrorism. IGI Global (2007)

\bibitem{schlenker2018AAMAS}
Schlenker, A., Thakoor, O., Xu, H., Fang, F., Tambe, M., Tran-Thanh, L.,
  Vayanos, P., Vorobeychik, Y.: Deceiving cyber adversaries: A game theoretic
  approach. In: AAMAS (2018)

\bibitem{shamsi2014ACM}
Shamsi, Z., Nandwani, A., Leonard, D., Loguinov, D.: Hershel: single-packet os
  fingerprinting. In: ACM SIGMETRICS Performance Evaluation Review (2014)

\bibitem{sinha2016AAMAS}
Sinha, A., Kar, D., Tambe, M.: Learning adversary behavior in security games: A
  pac model perspective. In: AAMAS (2016)

\bibitem{spitzner2003IEEE}
Spitzner, L.: The honeynet project: Trapping the hackers. IEEE Security \&
  Privacy  (2003)

\bibitem{stancu2012fractional}
Stancu-Minasian, I.M.: Fractional programming: theory, methods and
  applications, vol.~409. Springer Science \& Business Media (2012)

\bibitem{wang2018Gamesec}
Wang, W., Zeng, B.: A two-stage deception game for network defense. In: GameSec
  (2018)

\bibitem{yang2014AAMAS}
Yang, R., Ford, B., Tambe, M., Lemieux, A.: Adaptive resource allocation for
  wildlife protection against illegal poachers. In: AAMAS (2014)

\bibitem{yang2012AAMAS}
Yang, R., Ordonez, F., Tambe, M.: Computing optimal strategy against quantal
  response in security games. In: AAMAS (2012)

\bibitem{yin2014AAAI}
Yin, Y., An, B., Vorobeychik, Y., Zhuang, J.: Optimal deceptive strategies in
  security games: A preliminary study. In: AAAI Symposium on Applied
  Computational Game Theory (2014)

\end{thebibliography}
\newpage
\appendix

\bigskip

\section{Deferred Algorithms}
\label{app:algs}
We show the MILP formulation for the mathematical program $\mathcal{MP}1$. We use $M_c \subseteq M$ to denote the set of continuous features, and $M_d = M - M_c$ denotes the set of discrete features. For discrete feature $k \in M_d$, we assume that $\eta_{ik}$ and budget $B$ have been processed such that Constraint~\eqref{geneq2} has been modified to 
\[\sum_{i \in N} \left(\sum_{k \in M_c} \eta_{ik} |x_{ik} - \hat x_{ik}| + \sum_{k \in M_d} \eta_{ik} x_{ik}\right) \leq B.
\]
This transformation based on $\hat x_{ik} \in \{0,1\}$ simplifies our presentation below.
\begin{align}
    & \max_{b,d,g,h,q,s,t,v,y} \quad  \sum_{i \in N} t_i \label{milp2eq1}\\
    & s.t. \quad  t_i = v e^{-2W} + \sum_{l} \gamma_{l} (v\epsilon - s_{il}) & \label{milp2eq2}\\
    & \sum_{k \in M_c} w_k q_{ik} + \sum_{k \in M_d} w_k b_{ik} - W v = -\sum_{l} s_{il} & \label{milp2eq3}\\
    & h_{ik} \geq q_{ik} - \hat x_{ik} v, h_{ik} \geq \hat x_{ik} v - q_{ik} & \hspace{-1cm}\forall k \in M_c \label{milp2eq4}\\
    &  \sum_{i \in N} \left(\sum_{k \in M_d} \eta_{ik} b_{ik} + \sum_{k \in M_c} \eta_{ik} h_{ik}\right) \leq B v \label{milp2eq5}\\
    & \epsilon g_{il} \leq s_{il}, s_{i(l+1)} \leq \epsilon g_{il} & \forall l \label{milp2eq6}\\
    & s_{il} \leq v\epsilon &\forall l \label{milp2eq7} \\
    & g_{il} \leq v, g_{il} \leq Z y_{il}, g_{il} \geq v - Z(1-y_{il}) &\forall l \label{milp2eq8} \\
    & b_{ik} \leq v, b_{ik} \leq Z d_{ik}, b_{il} \geq v - Z(1-d_{ik}) &\hspace{-1cm}\forall k \in M_d \label{milp2eq9} \\
    & q_{ik} \in [(\hat x_{ik} - \tau_{ik})v, (\hat x_{ik} + \tau_{ik})v] \cap [0,1]&\hspace{-1.0cm}\forall k \in M_c \label{milp2eq10} \\
    & \sum_{i \in N} u_i t_i = 1 &\label{milp2eq11}\\
    & \text{Categorical constraints} &\label{milp2eq12} \\
    & t_i, v, s_{il}, q_{ik}, h_{ik}, g_{il} \geq 0, y_{il} \in \{0,1\}& \hspace{-1.2cm}\forall k \in M_c, \forall l\label{milp2eq13}\\
    & b_{ik} \geq 0, d_{ik} \in \{0,1\} &\hspace{-1cm} \forall k \in M_d
\end{align}
We establish the variables in the MILP above with the FDP variables as below.
\begin{align}
   t_i &= \frac{f_i}{\sum_{i \in N} f_i u_i},& v &= \frac{1}{\sum_{i \in N} f_i u_i} &\\
 h_{ik} &= \frac{|x_{ik} - \hat x_{ik}|}{\sum_{i \in N} f_i u_i},& q_{ik} &= \frac{x_{ik}}{\sum_{i \in N} f_i u_i}, &\forall k \in M_c\\
 d_{ik} &= x_{ik},& b_{ik} &= \frac{x_{ik}}{\sum_{i \in N} f_i u_i}, &\forall k \in M_d\\
 s_{il} &= \frac{z_{il}}{\sum_{i \in N} f_i u_i},& g_{il} &= \frac{y_{il}}{\sum_{i \in N} f_i u_i}, & \forall l
\end{align}
All equations above involving index $i$ without summation should be interpreted as applying to all $i \in N$.

\begin{algorithm}
	Use gradient-based method to find $x^{max} \approx \arg\max_{x} f(x)$ and $x^{min} \approx \arg\min_{x} f(x)$.\\
	Sort the targets such that $u_1 \leq u_2 \leq \cdots \leq u_n$.\\
	Initialize $i=1, j=n$.\\
	\While{$i < j$ and budget $> 0$}{
		Let $x_i \leftarrow x^{max}$ if \\
        \If{Cost$(x_i \leftarrow x^{max}) \leq$ remaining budget}{
            $x_i \leftarrow x^{max}$, decrease the budget, $i = i + 1$.
        }
        \If{Cost$(x_j \leftarrow x^{min}) \leq$ remaining budget}{
            $x_j \leftarrow x^{min}$, decrease the budget, $j = j - 1$.
        }
	}
	\Return feature configuration $x$
	\caption{\textsc{Greedy}}\label{alg:greedy}
\end{algorithm}

\section{Exact Algorithms for Special Cases} \label{app:specialcase}
\subsection{Deception cost on discrete features} \label{app:specialcase:costdiscrete}
In our first attempt at exact algorithms, we assume the cost of deception is only associated with discrete features, i.e. $\eta_{ik} = 0$ if $k$ is a continuous feature. 

Recall that we use $M_c \subseteq M$ to denote the set of continuous features, and $M_d = M - M_c$ denotes the set of discrete features.
Suppose $x_i^d = (x_{ik})_{k \in M_d}$ and $x_i^c = (x_{ik})_{k \in M_c}$, and let $x_i = (x_i^d, x_i^c)$ be the observable features decomposed into discrete features and continuous features.
Our score function $f(x_i) = \exp\{\sum_{k \in M} w_k x_{ik}\}$ can be factorized into $f(x_i) = f_d(x_i^d) f_c(x_i^c)$, where $f_d, f_c$ are scores considering discrete and continuous features only, respectively.

Let $A^d_i = \{x_i^{d_j}: j = 1,\dots, k\}$ be the finite set of possible discrete observable feature combinations at target $i$. Each $x_i^{d_j} \in \{0,1\}^{m_d}$ is a $m_d$-dimensional vector, where $m_d = |M_d|$ is the number of discrete features in FDP.  Based on the hidden discrete features $\hat x^d_{i}$ and the feasible regions $A(\hat x_{ik})$, we can compute both $A^d_i$ and all possible scores $f^d_{ij} = f_d(x_i^{d_j})$. Similarly, we can compute the interval $[\alpha_{i}, \beta_{i}]$ in which the continuous score $f^c_{i}$ could possibly lie, since each continuous feature $x_{ik}$ can take value in an interval $A(\hat x_{ik})$. 

Subsequently, we formulate the following mathematical program. The binary variable $y_{ij} = 1$ if target $i$'s discrete observable features are the $j$-th combination in $A^d_i$, that is, $x_i^d = x_i^{d_j} \in A^d_i$. The cost $c_{ij}$ for setting the discrete observable features to $x_{i}^{d_j}$ could be computed accordingly.

\begin{align*}
    \min_{y, f^c} \qquad & \frac{\sum_{i \in N} \sum_{j =1}^k u_i f^d_{ij} y_{ij} f^c_{i}}{\sum_{i \in N} \sum_{j =1}^k f^d_{ij} y_{ij} f^c_{i}} & \\
    s.t. \qquad & \sum_{j =1}^k y_{ij} = 1 & \forall i \in N\\
    & \sum_{i \in N} \sum_{j =1}^k c_{ij} y_{ij} \leq B & \\
    & f^c_{i} \in [\alpha_{i}, \beta_{i}] & \forall i \in N\\
    & y_{ij} \in \{0,1\} & \forall i \in N, j \in [k]
\end{align*}
We may apply the same linearization method as before to obtain a MILP.

Solving this MILP yields the optimal discrete feature configuration, as well as the optimal scores $f_i^c$'s of the continuous features. One may then solve the system of linear equations $\ln f_i^c = \sum_{k \in M} w_k x_{ik}$ for $i \in N$ for the optimal continuous features. Since the feasible regions of the continuous features are connected, there exists at least one solution to the system of equations. We remark that when all features are continuous, the above approach finds the optimal feature configuration in polynomial time.

\subsection{No budget and feasibility constraints}   \label{app:specialcase:nobudgetfeasibility}
We present an efficient algorithm when the defender has no budget and feasibility constraints. 
\cite{schlenker2018AAMAS} proposed a greedy algorithm in a similar setting (with feasibility constraints), whose complexity is polynomial in the size of ``type space'', which is still exponential in the representation of FDP, not to mention that with a single continuous feature the size of our ``type space'' becomes infinite. Furthermore, the probabilistic behavior of the attacker in FDP makes their key strategy inapplicable.

Since the defender aims at minimizing her expected loss, one intuitive idea is to give the lowest score to the target with the highest loss $u_i$. In fact, we show below that the defender should configure the features at each target in only two possible ways: the ones which maximizes or minimizes the score. First, we assume that the defender's losses have been sorted in ascending order $u_1 \leq u_2 \leq \cdots \leq u_n$.
\begin{lemma} \label{lem:sort}
Let $x = (x_1,\dots, x_n)$ be an optimal observable feature configuration. There exists a permutation $\sigma$ on $N$ where $x^\sigma = (x_{\sigma(1)},\dots, x_{\sigma(n)})$ is also an optimal observable feature configuration, and
\[
f(x_{\sigma(1)}) \geq f(x_{\sigma(2)}) \geq \cdots \geq f(x_{\sigma(n)}).
\]
In particular, if $u_1 < u_2 < \cdots <u_n$, $\sigma$ can be the identity permutation.
\end{lemma}
\begin{proof}
We prove by contradiction. Suppose $i < j$ and $f(x_i) < f(x_j)$. We have
\[
\begin{split}
&(f(x_j) u_i + f( x_i) u_j) - (f( x_i) u_i + f(x_j) u_j) \\
&= (f( x_j) - f(x_i)) (u_i - u_j) \leq 0
\end{split}
\]
and the inequality is strict if $u_i < u_j$.
Thus, when $u_i < u_j$, if we swap the features on target $i$ and $j$, we would strictly improve the solution, which contradicts $x$ being optimal. When $u_i = u_j$, we could swap the observed features on node $i$ and $j$, and strictly decrease the number of score inversions.
\end{proof}

\begin{lemma} \label{lem:extreme}
There exists an optimal observable feature configuration $x = (x_1,\dots, x_n)$ such that, for some $j \in N-\{n\}$, if $i \leq j$, $f(x_i) = \max_{ x_i'} f( x_i')$; otherwise, $f(x_i) = \min_{x_i'} f(x_i')$.
\end{lemma}
\begin{proof}
Let $x$ be an optimal observable feature configuration. Fix a target $i \in N$. Consider an alternative configuration $\tilde x_i$ for target $i$, while keeping features of other targets unchanged. We have
\[
\begin{split}
    &\frac{f(x_i) u_i + \sum_{j \neq i} f(x_j) u_j}{f(x_i) + \sum_{j \neq i} f(x_j)} - \frac{f(\tilde x_i) u_i + \sum_{j \neq i} f(x_j) u_j}{f(\tilde x_i) + \sum_{j \neq i} f(x_j)}\\
    &\propto \left( \left(f(x_i) - f(\tilde x_i)\right) \left(\sum_{j \neq i} f(x_j) (u_i - u_j) \right) \right)\\
\end{split}
\]
Depending on its sign, we could improve the solution $x$ by making $f(\tilde x_i)$ greater or smaller than $f(x_i)$, and obviously we should take it to extreme by setting $f(\tilde x_i) = \max_{x_i'} f(x_i')$ or $f(\tilde x_i) = \min_{x_i'} f(x_i')$. Since we assumed $x$ is optimal, then we know $\tilde x = (\tilde x_i, x_{-i})$ is also optimal, with $f(\hat x_i)$ at an extreme value. By Lemma~\ref{lem:sort}, we could permute the features in $\tilde x$ so that the scores are in decreasing order. After applying the above argument repeatedly, the score of each target achieves either the maximum or minimum score possible. Therefore, there exists some $j \in N-\{n\}$, such that 
\[
\begin{split}
&\max_{x_i'} f(x_i') = f(x_1) = \cdots = f(x_j) \\
&\geq f(x_{j+1}) = \cdots = f(x_n) = \min_{x_i'} f(x_i')
\end{split}
\]
\end{proof}

\begin{theorem} \label{thm:easiest}
The optimal feature configuration can be found in $O(n\log n + m)$ time.
\end{theorem}
\begin{proof}
We can do an exhaustive search on the ``cut-off'' node $j$ in Lemma~\ref{lem:extreme}. With bookkeeping, the search can be done in $O(n)$ time. Since $f$ is monotone in each observable feature, the maximum and minimum score can be found in $O(m)$ time. Sorting the targets' losses takes $O(n\log n)$ time.
\end{proof}

\section{Additional Experiments}
\label{app:exp}
\begin{figure}[t]
     \subfloat[Learning, classical score function]{\includegraphics[width=0.24\columnwidth]{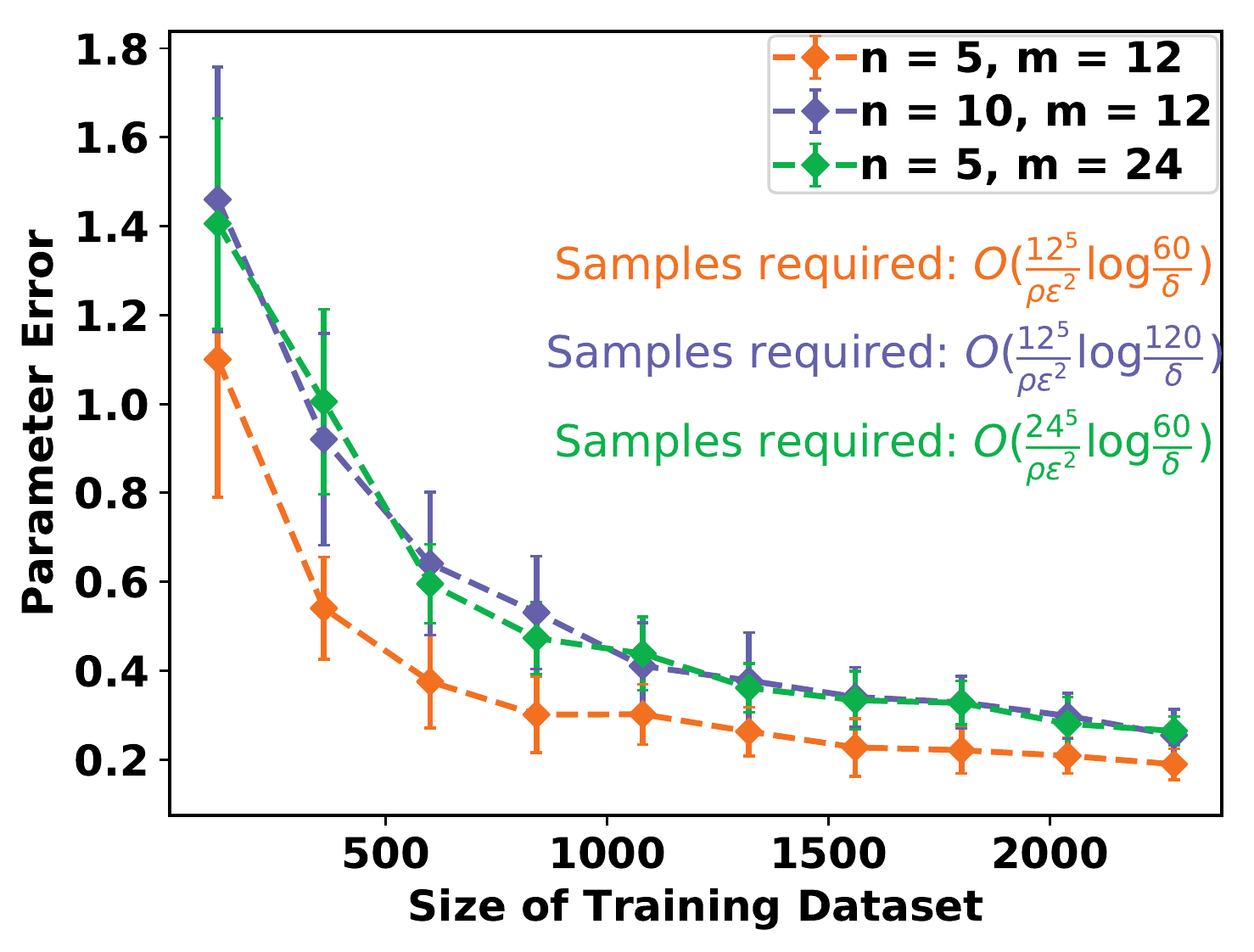} 
        \label{fig:learn_simple_param}} 
     \subfloat[Learning, NN-3 score function]{\includegraphics[width=0.24\columnwidth]{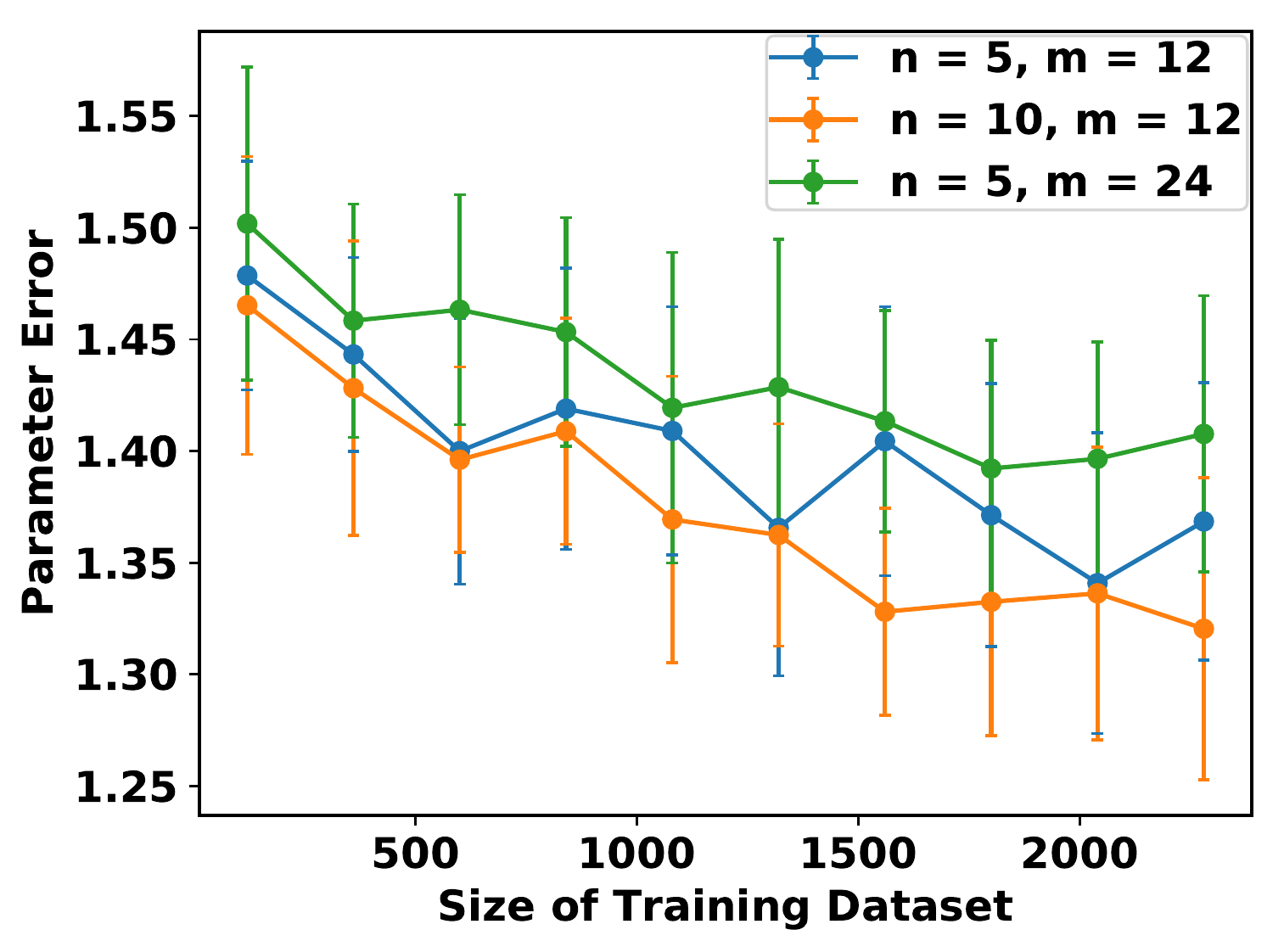}
        \label{fig:learn_complex_param}}
     \subfloat[Planning with classical score function, $n=10$]{\includegraphics[width=0.24\columnwidth]{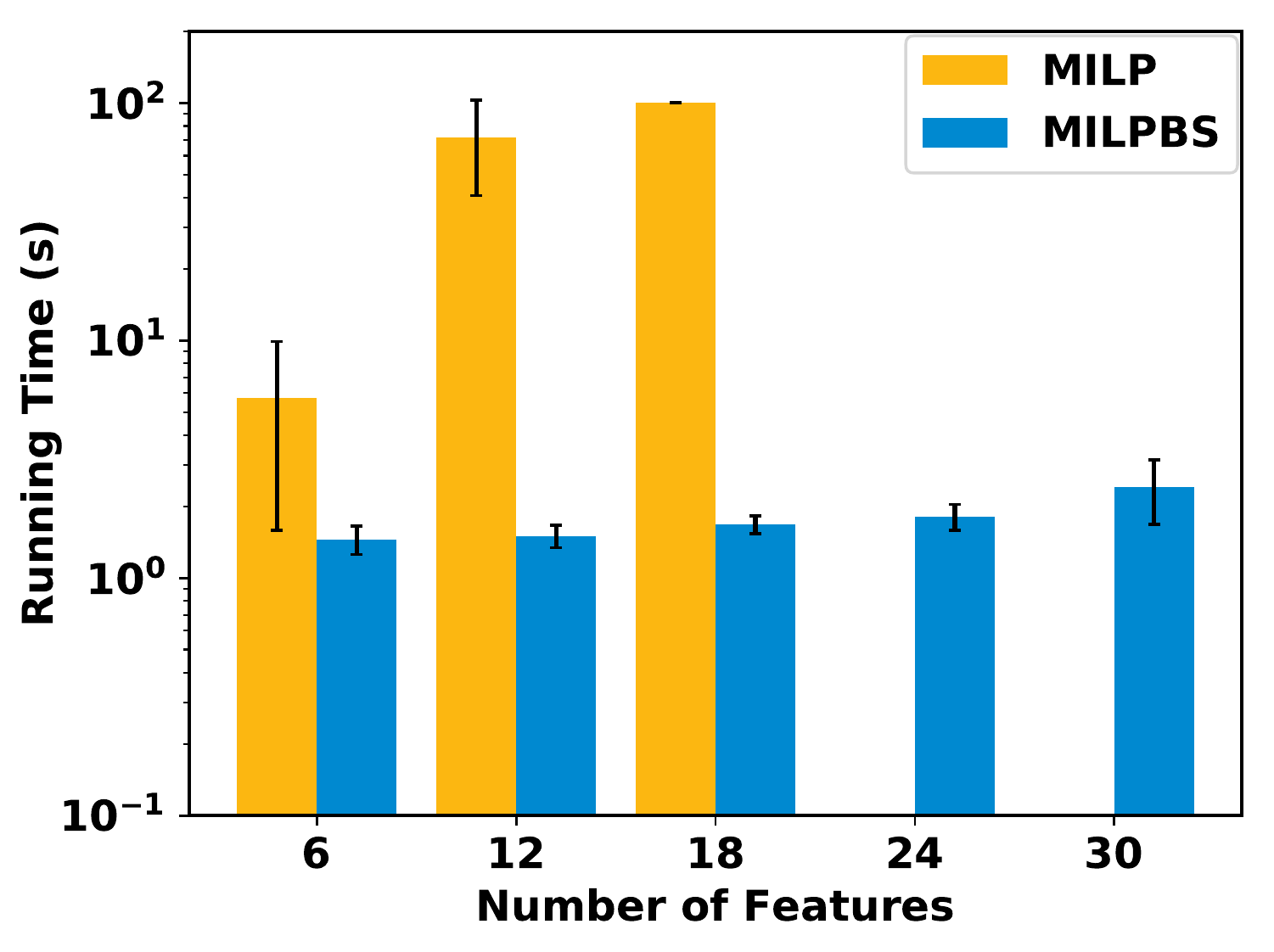}%
        \label{fig:plan_simple_general_K}}
    \subfloat[Planning with NN-3 score function, $n=10$]{\includegraphics[width=0.24\columnwidth]{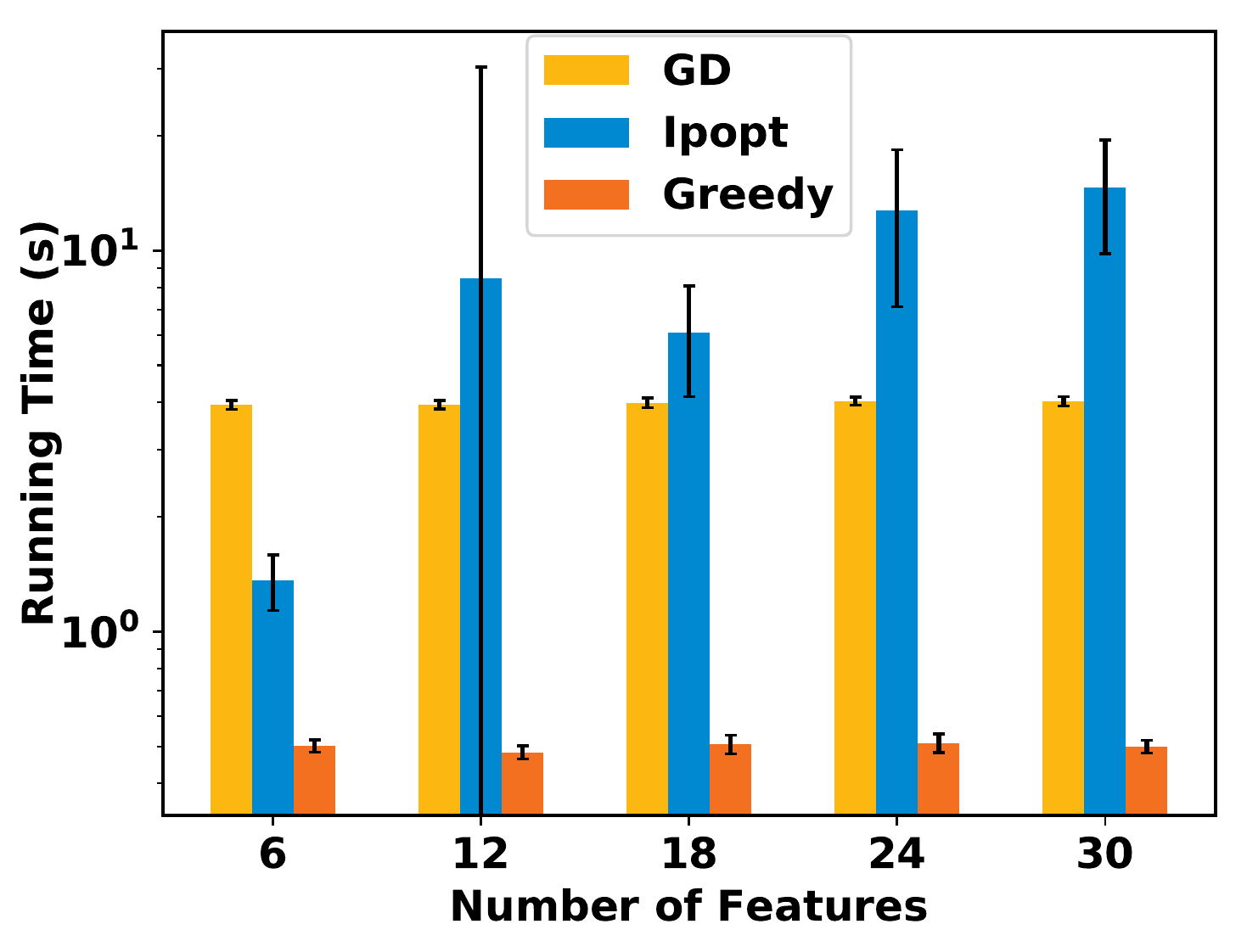}%
        \label{fig:plan_complex_runtime_K}}\\
     \subfloat[Planning with NN-3 score function, $n=10$]{\includegraphics[width=0.24\columnwidth]{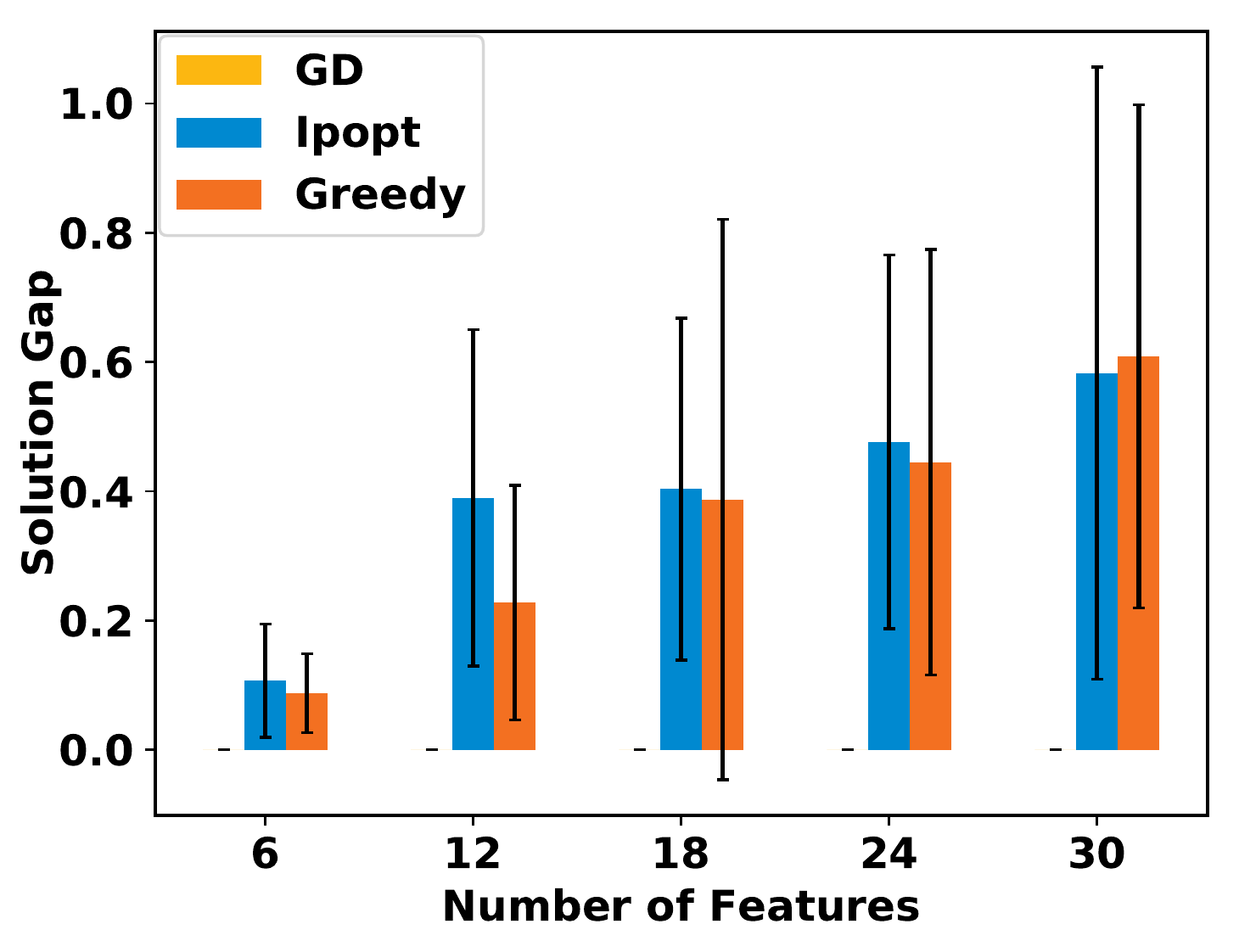}
        \label{fig:plan_complex_error_K}}
    \subfloat[Learning + Planning, NN-3 score function, $m=12$]{\includegraphics[width=0.24\columnwidth]{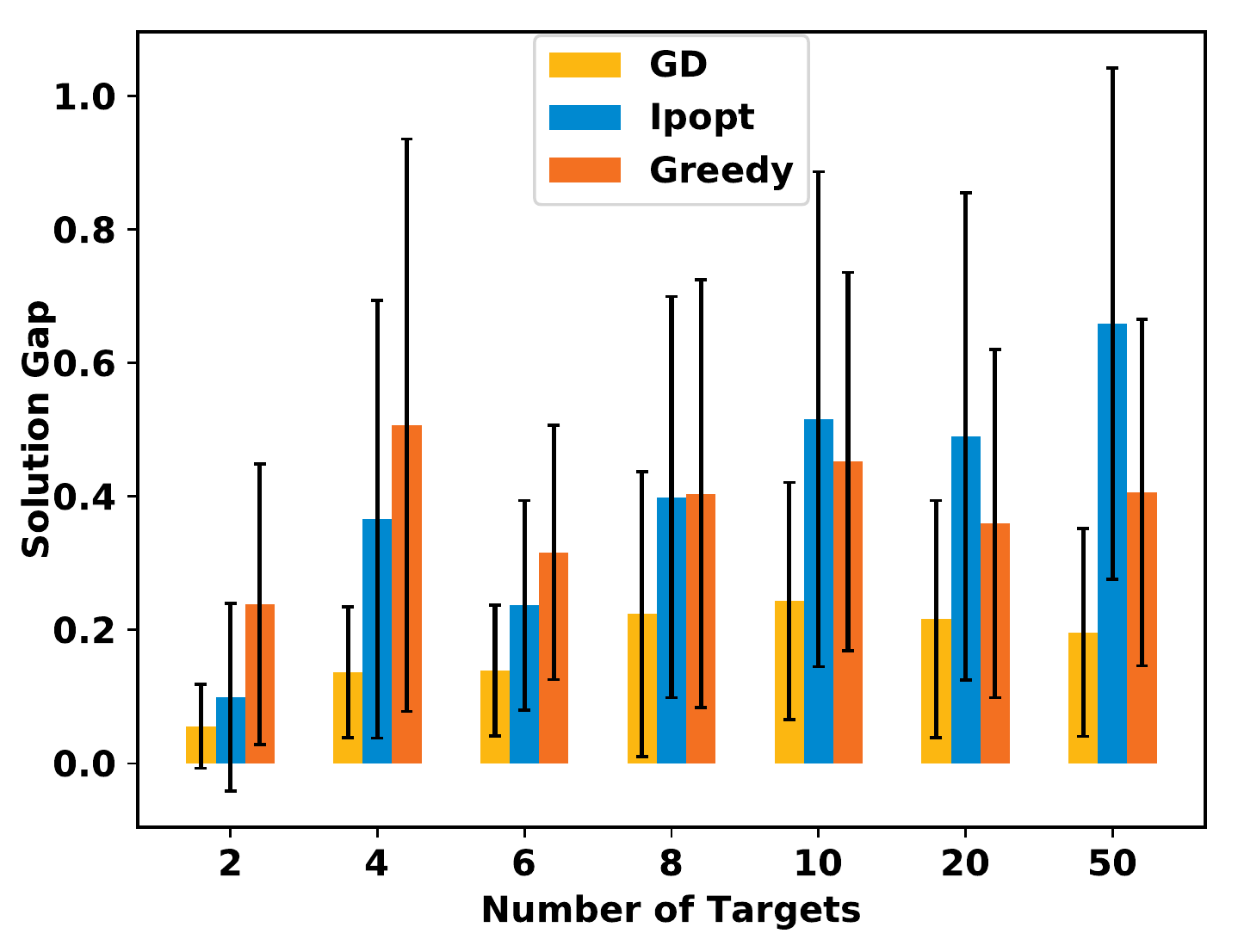}
        \label{fig:learn_plan_complex_N}}
        \subfloat[Continuous features]{\includegraphics[width=0.24\columnwidth]{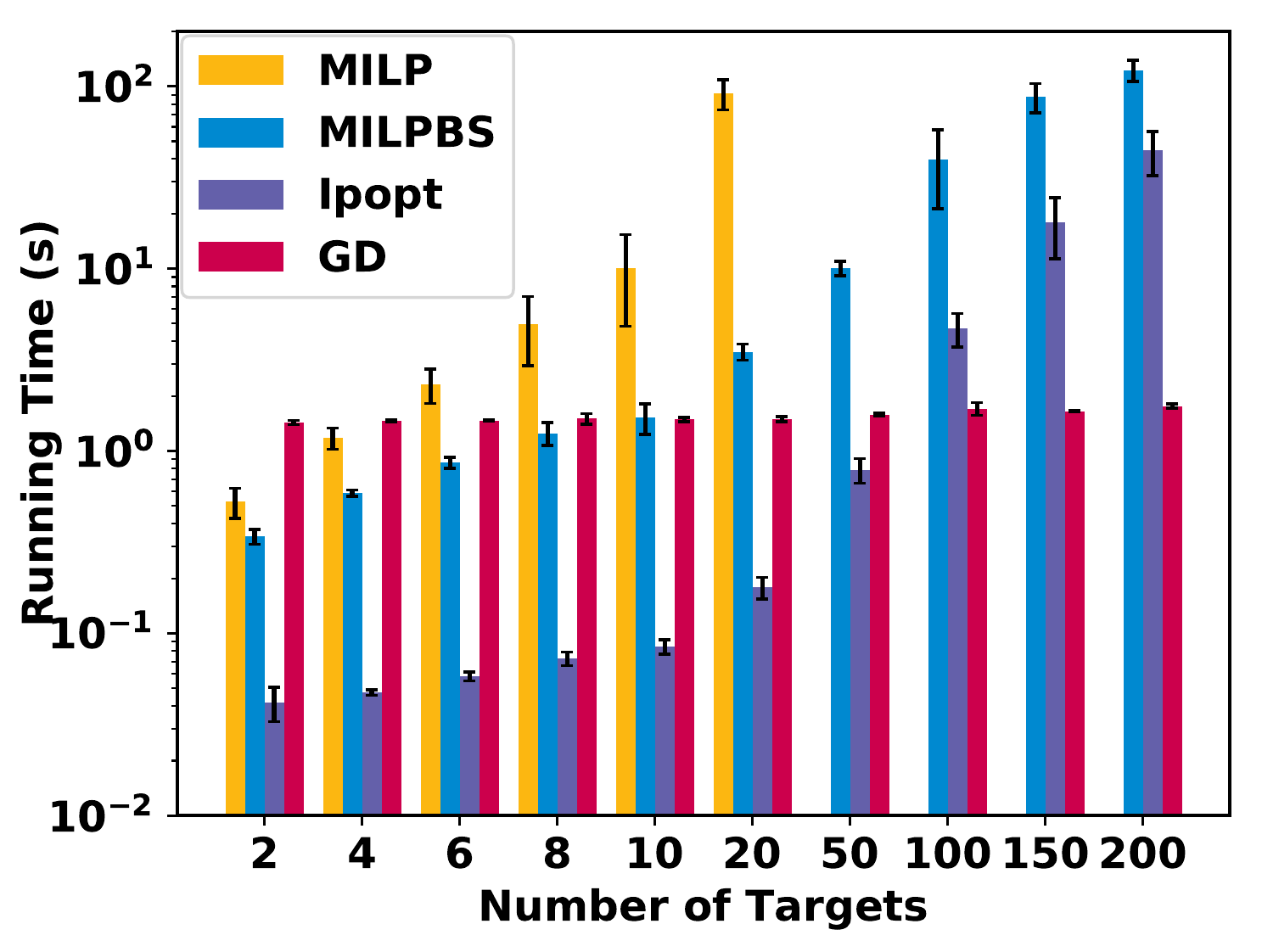}%
        \label{fig:plan_simple_continuous}} 
     \subfloat[Continuous features]{\includegraphics[width=0.24\columnwidth]{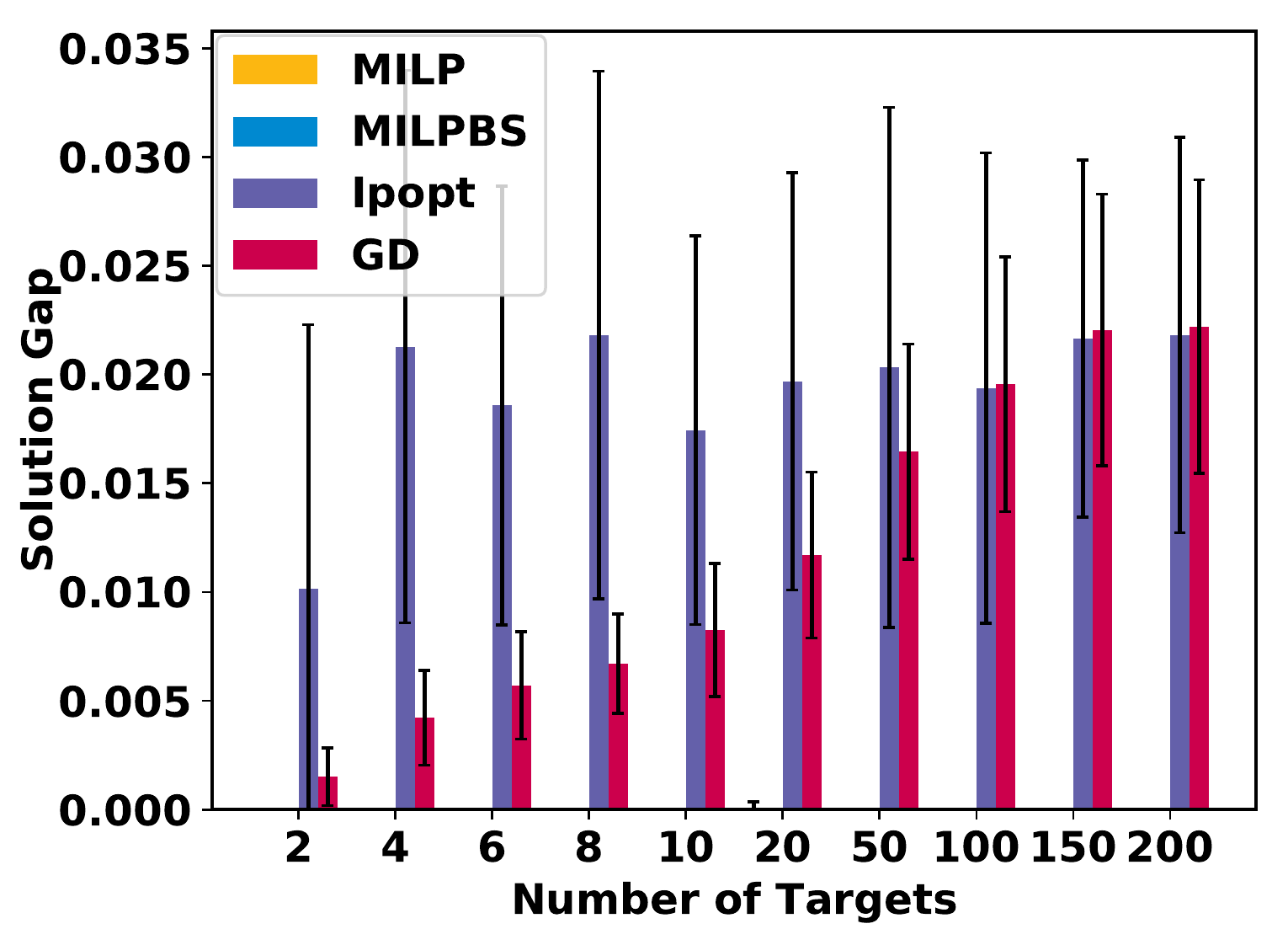}%
        \label{fig:plan_simple_continuous_gap}} \\
    \subfloat[Cost on discrete features, classical score function, $m=12$]{\includegraphics[width=0.24\columnwidth]{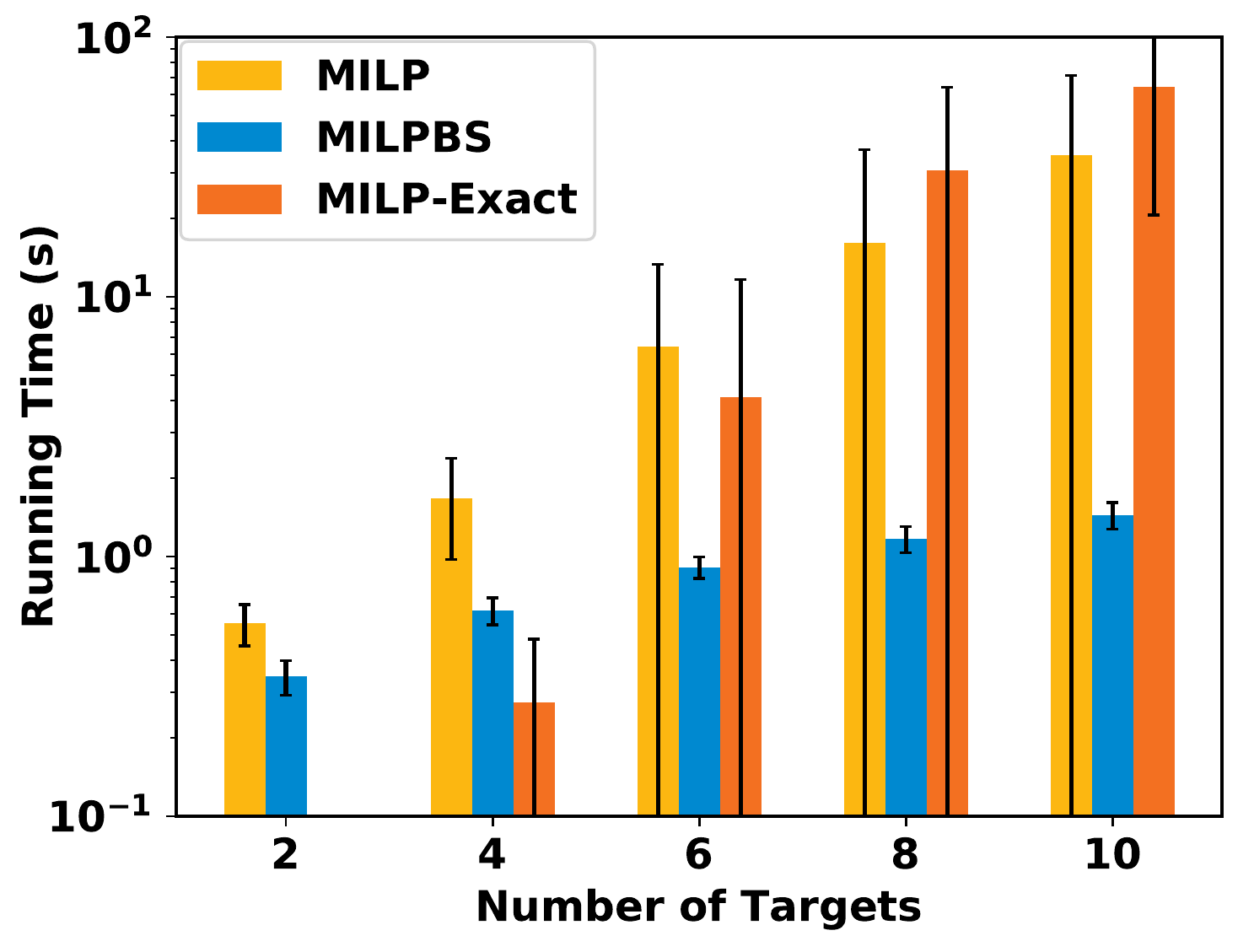}%
        \label{fig:plan_simple_cost_discrete}} 
    \subfloat[No budget and feasibility constraints, classical score function, $m=12$]{\includegraphics[width=0.24\columnwidth]{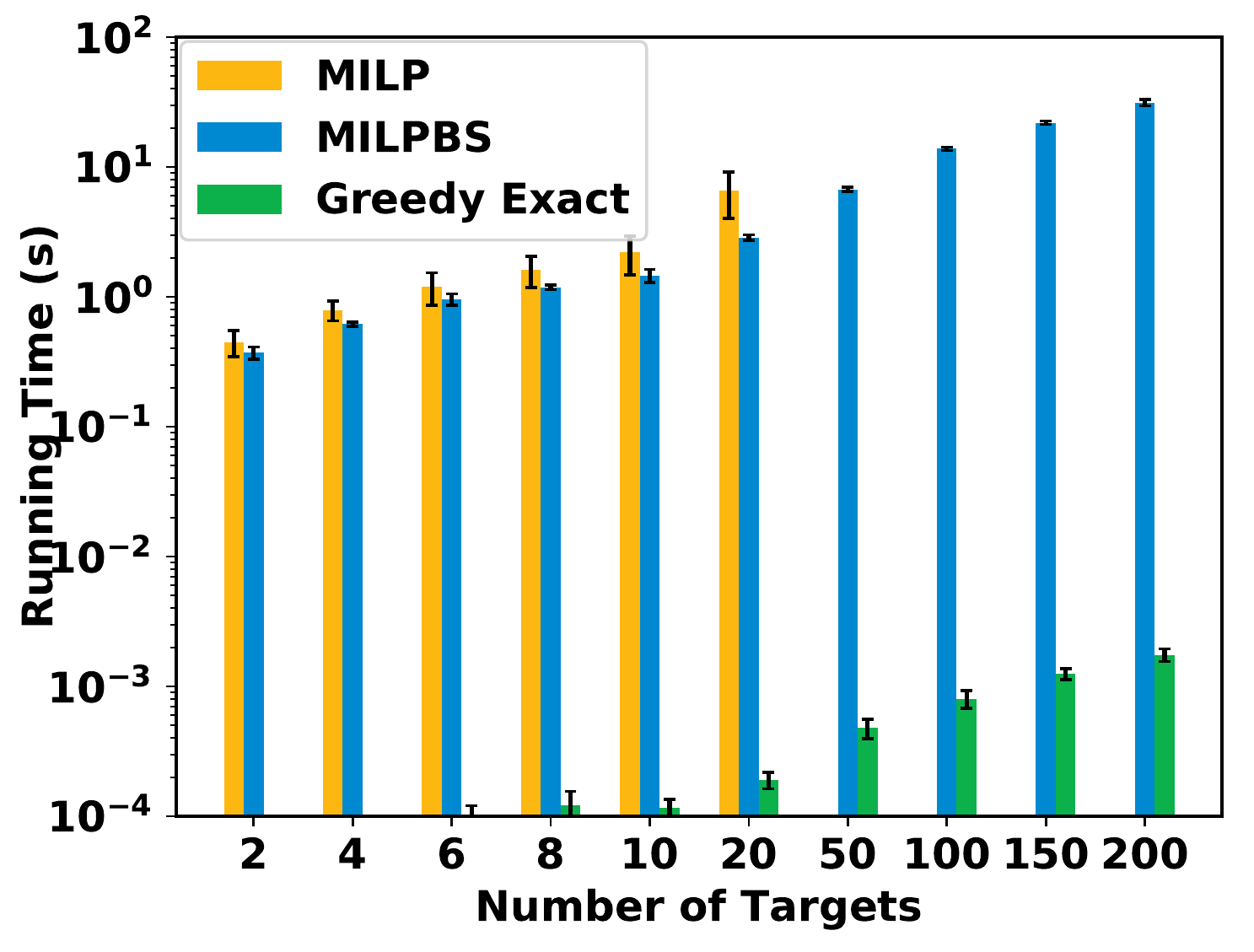}%
        \label{fig:plan_simple_nocost}} 
    \caption{Additional experimental results}
    \label{fig:all_appendix}
\end{figure}
\subsection{Experiments in the main text}

\paragraph{Learning}
In addition to the mean total variation distance reported in the main text, we present another metric to measure the performance of learning. 
We consider $|\hat \theta - \theta|$, the $L_1$ error in the score function parameter $\theta$, which directly relates to the sample complexity bound in Theorem~\ref{thm:samplecomplexity}. Since the dimension of $\theta$ depends on the number of features $k$ and other factors, we consider the $L_1$ error divided by the number of parameters and report this metric in Fig.~\ref{fig:learn_simple_param} and Fig.~\ref{fig:learn_complex_param}. 

In Fig~\ref{fig:learn_simple_param}, the $L_1$ error of CF decreases as the sample size increases. The complexity bounds in Theorem~\ref{thm:samplecomplexity} are marked in Fig~\ref{fig:learn_simple_param}. We need much fewer samples in practice to achieve a small learning error.
For NN-3 score function, the learning error is larger as shown in Fig.~\ref{fig:learn_complex_param}, even though Fig.~\ref{fig:learn_complex} in the main text shows the total variation distance is small. This suggests that the loss surface for the NN-3 score function is more complex. The solution gap in Fig.~\ref{fig:learn_plan_complex_D} is much larger than that in Fig.~\ref{fig:learn_plan_simple_D}, which can partly be explained by the fact that at the same level of TV error, the learned classical score function score function is closer to the ground truth than the NN-3 score function, and thus performs better in the learning and planning pipeline.

\paragraph{Planning}
In the main text, we showed how the number of targets effect the running time and solution gap. In 
In Fig.~\ref{fig:plan_simple_general_K}-~\ref{fig:plan_complex_error_K}, we show how the number of features affect these two metrics. For the classical score function, Fig.~\ref{fig:plan_simple_general_K} shows that the running time for MILP increases with the number of features, and MILPBS is much more scalable. 
Fig.~\ref{fig:plan_complex_runtime_K} and \ref{fig:plan_complex_error_K} show the running time and solution gap fixing $n = 10$.
The running time of GD and \textsc{Greedy} does not change much across different problem sizes, yet Ipopt runs slower than the former two on most problem instances.
GD also has smaller solution gap than Ipopt and \textsc{Greedy} on most instances.

\paragraph{Combining learning and planning}
Given enough data, Fig.~\ref{fig:learn_plan_complex_N} shows that GD can achieve a solution gap below 0.2 with as many as 50 targets.

\subsection{Experiments for the special cases}
We present the performance of our algorithms on some special cases of FDP, as studied in Appendix~\ref{app:specialcase}. 

When all features are continuous, in addition to our MILP, we may use non-convex solver or GD as a heuristic to find optimal feature configurations. Fig.~\ref{fig:plan_simple_continuous} shows that these two heuristics scale better than the approximation algorithms.
In Fig.~\ref{fig:plan_complex_error_N} and Fig.~\ref{fig:plan_complex_error_K}, we showed that GD demonstrates the best solution quality among the heuristics on NN-3 score functions. A natural question to ask is if GD is in practice close to exact. In Fig.~\ref{fig:plan_simple_continuous_gap}, we show that at least when we have a single-layer score function, GD solution is not far from optimal, though its solution deteriorates as the problem size grows. Non-convex solver yields a small solution gap as well. 

When deception cost is only associated with discrete features, we presented an exact MILP formulation in Appendix~\ref{app:specialcase:costdiscrete}. Fig.~\ref{fig:plan_simple_cost_discrete} shows that it is especially efficient on very small problems. Yet as the problem size grows its efficiency decreases quickly.

Finally, in Appendix~\ref{app:specialcase:nobudgetfeasibility}, we proposed a $O(n\log n + m)$ time algorithm for FDP without budget and feasibility constraints. Indeed, as shown in Fig.~\ref{fig:plan_simple_nocost}, the algorithm's running time is several magnitudes less than the MILP-based approaches.

\section{Experiment Parameters and Hyper-parameters}
\label{app:param}
\paragraph{NN-3 score function architecture}
The NN-3 score function has input layer of size $m \times 24$, second layer $24 \times 12$, and third layer $12 \times 1$. The first and second layers are followed by a tanh activation, and the last layer is followed by an exponential function. The neural network parameters are initialized uniformly at random in $[-0.5, 0.5]$. We use this network architecture for all of our experiments.

\paragraph{FDP parameters for classical score function}
We detail in Table~\ref{tab:distribution_simple} the parameter distributions used in the planning and combined learning and planning experiments, when the adversary assumes the single-layer score function.
These distributions apply to the results shown in Fig.~\ref{fig:plan_simple_general_N}, \ref{fig:plan_simple_general_K}, \ref{fig:learn_plan_simple_D}, \ref{fig:learn_plan_simple_N}. 
\begin{table}[t]
    \centering
    \begin{tabular}{cccc}
        \cline{1-4}
         \multicolumn{2}{c}{Discrete feature $k \in M_d$} & \multicolumn{2}{c}{Continuous feature $k \in M_c$}\\
        \cline{1-4}
        Variable & Distribution & Variable & Distribution\\
        \cline{1-4}
         $|M_d|$ & $2m/3$ & $|M_c|$ & $m/3$\\
         $\eta_{ik}$ & $U(-3, 3)$ & $\eta_{ik}$ & $U(0, 3)$\\
         $\tau_{ik}$ & N/A & $\tau_{ik}$ & $U(0, 0.25)$\\
         $\hat x_{ik}$ & $U\{0, 1\}$ & $\hat x_{ik}$ & $U(0, 1)$\\
         $u_i$ & \multicolumn{3}{c} {$U(0,1)$} \\
         \cline{1-4}
         Variable & \multicolumn{3}{c} {Distribution} \\
         \cline{1-4}
        $B$ &\multicolumn{3}{c} {$U(0, 0.2 C_{\max})$}\\
        $C_{\max}$ &\multicolumn{3}{c} {$\sum\limits_{i\in N} \sum\limits_{k \in M_c} \eta_{ik} \min(\hat x_{ik}, 1-\hat x_{ik}, \tau_{ik}) + \sum\limits_{k \in M_d} \eta_{ik}$}\\
        \hline
    \end{tabular}
    \caption{FDP parameter distributions for experiments on classical attacker score function. Used in Fig.~\ref{fig:plan_simple_general_N}, \ref{fig:plan_simple_general_K}, \ref{fig:learn_plan_simple_D}, \ref{fig:learn_plan_simple_N}} \label{tab:distribution_simple}
\end{table}

\paragraph{FDP parameters for NN-3 score function}
We detail in Table~\ref{tab:distribution_complex} the parameter distributions used in the planning and combined learning and planning experiments, when the adversary assumes the NN-3 score function. These distributions apply to the results shown in Fig.~\ref{fig:plan_complex_runtime_N},\ref{fig:plan_complex_error_N}, \ref{fig:plan_complex_runtime_K}, \ref{fig:plan_complex_error_K},\ref{fig:learn_plan_complex_D}, \ref{fig:learn_plan_complex_N}. 

\paragraph{Hyper-parameters for learning}
Table~\ref{tab:hyperparam_learning} shows the hyper-parameters we used in learning the attacker's score function $f$.

\begin{table}[t]
\begin{center}
    \begin{tabular}{ l  l }
    \hline
    Variable & Distribution \\ \hline
    $\eta_{ik}$ & $U(0, 1)$\\ 
    $\tau_{ik}$ & $1$\\ 
    $\hat x_{ik}$ &  $U(0,1)$\\ 
    $u_i$ & $U(0, 1)$\\ 
    $B$ & $U(0, 0.2nm)$\\ \hline
    \end{tabular}
\end{center}
\caption{FDP parameter distributions for experiments on NN-3 attacker score function. Used in Fig.~\ref{fig:plan_complex_runtime_N},\ref{fig:plan_complex_error_N}, \ref{fig:plan_complex_runtime_K}, \ref{fig:plan_complex_error_K},\ref{fig:learn_plan_complex_D}, \ref{fig:learn_plan_complex_N}} \label{tab:distribution_complex}
\end{table}

\begin{table}
\begin{center}
    \begin{tabular}{ l  l l l}
    \hline
    Parameter  & Fig~\ref{fig:learn_plan_complex_D} ($|D_{train}| >10000$),~\ref{fig:learn_plan_complex_N} & Fig.~\ref{fig:learn_plan_simple_N} & All other experiments\\ \hline
    Learning rate  & $\{\text{1e-3, 1e-2, 1e-1}\} \to \text{1e-1}$ & $\{\text{1e-3, 1e-2, 1e-1}\} \to \text{1e-1}$ & $\{\text{1e-3, 1e-2, 1e-1}\} \to \text{1e-1}$\\ 
    Number of epochs  & $\{\text{20, 30, 60}\} \to 30$ & $\{\text{20, 30, 60}\} \to 30$ & $\{\text{10, 20, 40}\} \to 20$\\ 
    Steps per epoch  & $\{\text{20, 30, 40}\} \to 30$ & 12 &  $\{\text{10, 20}\} \to 10$\\ 
    Batch size  & $\{\text{120, 600, 5000}\} \to 5000$ & $\{\text{120, 600, 5000}\} \to 5000$ & $|D_{train}|/$Number of epochs\\ \hline
    \end{tabular}
\end{center}
\caption{Hyper-parameters for the experiments. The values between the braces are the ones we tested. The values after the arrows are the ones we used in generating the results.} \label{tab:hyperparam_learning}
\end{table}

\end{document}